\documentclass[runningheads]{llncs}

\PassOptionsToPackage{dvipsnames,svgnames,x11names}{xcolor}

\usepackage[final,year=2026]{eccv}
\usepackage{eccvabbrv}

\usepackage{hyperref}

\usepackage{microtype}
\usepackage{graphicx}
\usepackage{booktabs} % for professional tables

\usepackage[textsize=tiny]{todonotes}

\usepackage{algorithm}
\usepackage{algorithmic}

\usepackage{enumitem}

\usepackage{kotex}

\usepackage{tikz}
\usetikzlibrary{arrows.meta,positioning,shapes,calc}

\usepackage{pgfplots}
\pgfplotsset{compat=1.18}

\usepackage{placeins}

\usepackage{float}

\usepackage{wrapfig}

\usepackage{multirow}

\usepackage{siunitx}
\usepackage{mathtools}
\usepackage{bm}

\spnewtheorem{assumption}[theorem]{Assumption}{\bfseries}{\itshape}

\newcommand{\xtheta}{x_\theta}
\newcommand{\epstheta}{\epsilon_\theta}
\newcommand{\vtheta}{v_\theta}

\newcommand{\zt}{z_t}
\newcommand{\xhat}{\hat{x}}

\newcommand{\E}{\mathbb{E}}
\newcommand{\Var}{\mathrm{Var}}

\newcommand{\N}{\mathcal{N}}

\newcommand{\R}{\mathbb{R}}
\newcommand{\manifold}{\mathcal{M}}
\newcommand{\energy}{\mathcal{E}}

\newcommand{\norm}[1]{\left\|#1\right\|}
\newcommand{\abs}[1]{\left|#1\right|}
\newcommand{\inner}[2]{\left\langle #1, #2 \right\rangle}

\newcommand{\grad}{\nabla}

\newcommand{\jacobian}{\mathbf{J}}

\newcommand{\bigO}{\mathcal{O}}

\newcommand{\mI}{\mathbf{I}}

\newcommand{\suppref}[1]{Appendix~\ref{#1}}

\ifdefined\MAINONLY
  \usepackage{xr-hyper}
\fi

\begin{document}

\title{Not All Prediction Targets Keep Training-Free
Diffusion Guidance on the Manifold}
\titlerunning{Not All Prediction Targets Keep TFG on the Manifold}

% Title-page footnote symbols: redefine llncs \@fnsymbol so 1=star (equal), 2=dagger (corresponding)
\makeatletter
\renewcommand\@fnsymbol[1]{\ensuremath{\ifcase#1\or \star\or \dagger\or \ddagger\or \mathsection\or \mathparagraph\fi}}
\makeatother
\newcommand\samethanks[1][\value{footnote}]{\footnotemark[#1]}
\author{Yunsung Lee\inst{1}\thanks{Equal contribution.} \and
Hyeongmin Lee\inst{2}\samethanks\thanks{Corresponding author.}}
\authorrunning{Y. Lee and H. Lee}
\institute{MAUM.AI, Republic of Korea\\
\email{sung@maum.ai} \and
Seoul National University of Science and Technology, Republic of Korea\\
\email{hyeongmin.lee@seoultech.ac.kr}}
\maketitle
% restore numeric footnotes for the body
\renewcommand{\thefootnote}{\arabic{footnote}}
\setcounter{footnote}{0}

% =============================================================================
% ABSTRACT
% =============================================================================
% =============================================================================
% ABSTRACT
% =============================================================================
\begin{abstract}
Training-free guidance (TFG) steers a pretrained diffusion model toward a desired attribute at inference. To be effective, this guidance must be applied from the earliest, high-noise steps of sampling. Because its objective (a classifier or energy) is defined on clean images, $\epsilon$- and $v$-prediction models must first estimate the clean image $\xhat$ from the noisy state at each step, and the accuracy of that estimate determines how easily guidance drifts off the data manifold. $x$-prediction, a recent alternative, outputs the clean image directly, removing this source of error even at high noise. This is our motivation.
We provide a theoretical analysis of how each prediction target shapes this accuracy, and introduce guided-class FID (Child FID), a metric that exposes the manifold damage standard evaluation misses. Experiments on a new fine-grained bird benchmark and on style transfer confirm that $x$-prediction keeps guided samples on the manifold most reliably, making it the strongest foundation for training-free guidance. Code is available at \url{https://github.com/ManLuML/on-manifold-tfg}.
\end{abstract}

% Teaser figure — declared immediately after abstract for page-1 placement
% =============================================================================
% TEASER FIGURE
% =============================================================================

\begin{figure}[!ht]
    \centering
    \begin{subfigure}[b]{0.35\textwidth}
        \includegraphics[width=\textwidth]{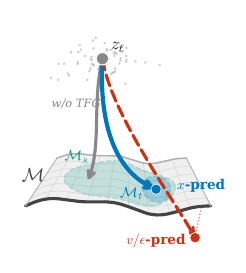}
        \caption{Guidance schematic}
        \label{fig:teaser_schematic}
    \end{subfigure}%
    \hspace{4pt}%
    \begin{subfigure}[b]{0.63\textwidth}
        \includegraphics[width=\textwidth]{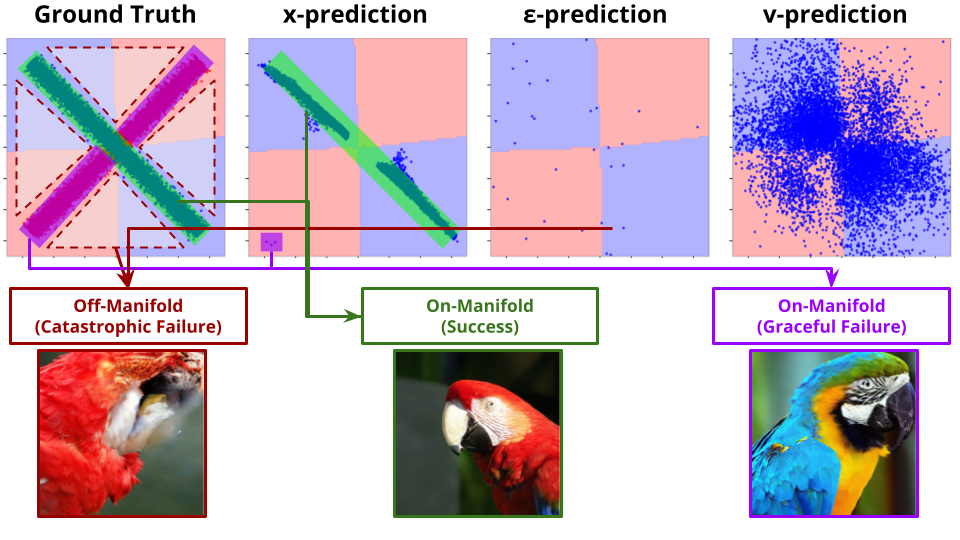}
        \caption{Crossed-lines ($D{=}512$) and ImageNet examples}
        \label{fig:teaser_crossed}
    \end{subfigure}
    \caption{
        \textbf{Prediction target determines whether training-free guidance stays on-manifold.}
        \textbf{(a)}~Without guidance, all targets denoise $\zt$ onto source class $\manifold_s \subset \manifold$. With TFG, $x$-prediction slides along $\manifold$ to target class $\manifold_t$; $v/\epsilon$-prediction departs $\manifold$.
        \textbf{(b)}~Crossed-lines ($D{=}512$, top) and ImageNet (bottom): $x$-prediction preserves structure; $\epsilon$-prediction collapses off-manifold with catastrophic artifacts.
    }
    \label{fig:teaser}
\end{figure}

% =============================================================================
% MAIN CONTENT
% =============================================================================
% =============================================================================
% SECTION 1: INTRODUCTION
% =============================================================================
\section{Introduction}
\label{sec:introduction}

% P1: Problem — catastrophic vs graceful failure
Training-free guidance (TFG)~\cite{chung2023dps,song2023lgd,ye2024tfg} steers diffusion models~\cite{ho2020ddpm,rombach2022ldm} toward desired properties without retraining, but strong guidance can push samples off the data manifold. The resulting \emph{catastrophic failures} (collapsed, distorted images) are qualitatively different from a \emph{graceful failure}, in which guidance misses the target class but the image remains a realistic sample from the learned distribution. A model whose worst case is graceful failure is fundamentally more dependable: even when guidance errs, the basic contract of generative modeling, producing plausible images, is preserved.

% P2: Why this went unnoticed — evaluation rewards off-manifold
Yet this distinction has gone unnoticed. Standard Validity (top-1 accuracy) rewards any sample the classifier accepts, whether on- or off-manifold, a blind spot shared by 15 of 17 recent TFG papers~\cite{shen2024tfgunderstanding} (\suppref{app:eval_survey}). When prior work maximises Validity under strong guidance~\cite{ye2024tfg}, it unknowingly selects off-manifold images fooling classifiers (akin to adversarial perturbations~\cite{stutz2019disentangling,nie2022diffpure}) rather than diverse samples of the target class. The evaluation does not merely fail to detect manifold departure; it actively encourages it.

% P3: Prediction target as root cause — thesis + research question
We trace the difference between catastrophic and graceful failure to a design decision predating any guidance algorithm: the prediction target. Three targets have been proposed: $\epsilon$-prediction (noise)~\cite{peebles2023dit}, $v$-prediction (velocity)~\cite{ma2024sit,chen2025pixelflow}, and $x$-prediction (clean data)~\cite{li2025jit}. Guidance operates on the clean-image estimate~$\xhat$: $\epsilon$- and $v$-prediction must recover it from the noisy state, whereas $x$-prediction outputs it directly. Under identical architecture and training, the three targets produce comparable generation quality, yet differ fundamentally in manifold preservation~\cite{li2025jit}: $x$-prediction succeeds where $\epsilon$-prediction fails catastrophically. Because TFG computes its guidance gradient through this estimate ($\nabla_{\zt} \energy(\xhat)$), the fidelity of~$\xhat$ directly controls guidance quality. We ask: \textit{does the prediction target's influence on manifold quality, demonstrated at training time~\cite{li2025jit}, extend to inference-time guidance?}

% P4: Theory Preview + Synthetic Evidence
We prove that prediction targets create a strict hierarchy of error amplification (\cref{prop:error_amplification}). The mechanism is the recovery formula: $\epsilon$-prediction divides by $t$, amplifying errors by $(1{-}t)/t$, a factor that diverges at high noise ($t \to 0$). In contrast, $v$-prediction attenuates errors by a bounded $(1{-}t)$ factor, and $x$-prediction introduces no amplification at all. These per-step errors compound across sampling, causing $\epsilon$-prediction's cumulative perturbation to diverge while $x$-prediction's remains bounded (\cref{prop:cumulative_error}). The gap further widens with ambient dimension~\cite{karras2022edm,jin2026kdiff} (\suppref{app:manifold_theory}). Controlled ablations confirm this hierarchy: in crossed-lines experiments (identical architecture and training, varying only the prediction target across $D \in \{2, 8, 32, 128, 512\}$), $x$-prediction maintains high on-manifold rate while $\epsilon$-prediction collapses (\cref{fig:teaser}, \cref{fig:crossed_lines_compact,fig:onmanifold_vs_dim}).

\begin{figure}[t]
\begin{minipage}[t]{0.40\textwidth}
    \centering
    \includegraphics[width=\textwidth]{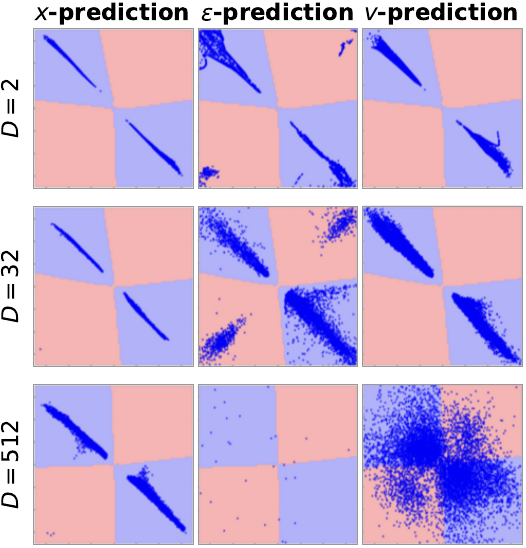}
    \caption{
        \textbf{Crossed-lines guided generation} ($s{=}10$, 100 steps). $\epsilon$-prediction collapses by $D{=}32$. Full grid in \suppref{app:experimental_protocols}.
    }
    \label{fig:crossed_lines_compact}
\end{minipage}
\hfill
\begin{minipage}[t]{0.57\textwidth}
    \centering
    \includegraphics[width=\textwidth]{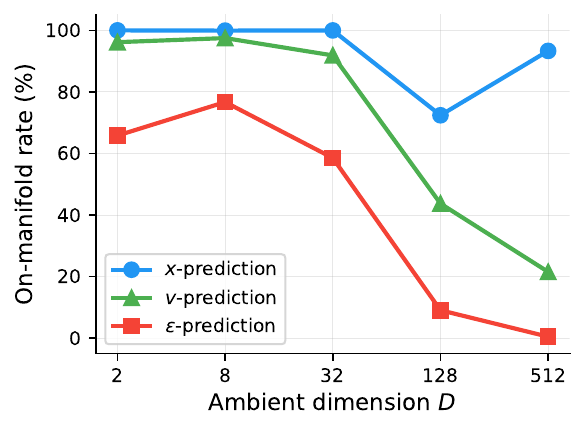}
    \caption{
        \textbf{On-manifold rate vs.\ ambient dimension} ($s{=}10$).
        $x$-prediction holds $>$93\% at $D{=}512$; $v$-prediction degrades to 21.5\%; $\epsilon$-prediction drops to 0.5\%. Data from \cref{tab:crossed_lines_full}.
    }
    \label{fig:onmanifold_vs_dim}
\end{minipage}
\end{figure}

% P5: Practical Evaluation + ImageNet Evidence
Validating this hierarchy at real scale demands three evaluation advances absent from prior TFG work: (i)~a fine-grained benchmark separating guidance and evaluation classifiers, (ii)~a manifold-aware metric, and (iii)~guidance-strength sweep plots replacing single-point comparisons. We construct a 143-species bird classification benchmark on ImageNet 256${\times}$256 and introduce Child FID (C-FID), FID between guided samples and the target species domain, sweeping guidance strength $\rho$ across Pareto frontiers (\cref{fig:rho_sweep}). We evaluate four pretrained Diffusion Transformers at comparable quality (FID${\approx}$2) across all three prediction targets~\cite{peebles2023dit,ma2024sit,li2025jit,chen2025pixelflow}. At matched classifier accuracy ($\approx$26.6\%), C-FID reveals a 5.2-point gap between $x$- and $\epsilon$-prediction (\textbf{32.9} vs.\ \textbf{38.1}), manifold damage invisible to standard evaluation. Qualitative analysis shows $\epsilon$-prediction achieving Validity via classifier-friendly patterns rather than diverse samples (\cref{fig:qualitative}). PixelFlow ($v$-prediction, pixel space) isolates prediction target as the decisive variable: its C-FID \emph{reverses} under strong guidance while JiT's continues decreasing (\cref{sec:results}).

% P6: Results Recap + Insight
Across controlled ablations and ImageNet-scale experiments alike, $x$-prediction yields the most stable behavior among the three targets for inference-time guidance. Among JiT variants (B/L/H), larger models achieve strictly better guidance Pareto frontiers, a \emph{guidance scaling} effect in which capacity improves both generation quality and guidance responsiveness. These findings establish prediction target selection as a first-order design decision for inference-time control.

% Contributions — keep the run-in header with its list (avoid an orphan at the page bottom)
\newpage
\paragraph{Contributions.}
This paper makes three contributions:
\begin{enumerate}[leftmargin=*,itemsep=2pt,parsep=0pt]
    \item \textbf{Theoretical framework.}
    We prove a strict error amplification hierarchy across prediction targets and show these errors compound into divergent trajectory perturbation for $\epsilon$-prediction while $x$-prediction's cumulative error remains bounded (\cref{prop:error_amplification,prop:cumulative_error}).

    \item \textbf{Manifold-aware evaluation protocol.}
    We introduce \emph{(i)}~a fine-grained bird classification benchmark (143 species, separate guidance and evaluation classifiers), \emph{(ii)}~Child FID (C-FID) to measure within-class realism, and \emph{(iii)}~guidance-strength Pareto sweeps replacing single-point comparisons.

    \item \textbf{Empirical validation.}
    In crossed-lines ablations, $x$-prediction maintains $>$93\% on-manifold rate while $\epsilon$-prediction collapses to $<$1\%. On ImageNet at matched Validity, C-FID reveals a 5.2-point gap between $x$- and $\epsilon$-prediction; PixelFlow's C-FID reversal confirms prediction target, not operating space, as the decisive factor (\cref{sec:results}).
\end{enumerate}

% =============================================================================
% SECTION 2: RELATED WORK (compressed for 14-page main body)
% =============================================================================
\section{Related Work}
\label{sec:related_work}

\paragraph{Prediction Targets in Diffusion Models.}
DDPM~\cite{ho2020ddpm} established $\epsilon$-prediction as the default; score-based models~\cite{song2021scorebased} gave an equivalent view via Tweedie's formula~\cite{robbins1956empirical,efron2011tweedie}. Salimans and Ho~\cite{salimans2022progressive} introduced $v$-prediction for improved stability. State-of-the-art Diffusion Transformers achieve comparable quality across all three targets: DiT-XL~\cite{peebles2023dit} ($\epsilon$), SiT-XL~\cite{ma2024sit} ($v$), and JiT-G~\cite{li2025jit} ($x$). Prior work observed $\epsilon$-prediction's training-time error amplification~\cite{karras2022edm,hang2023minsnr}; Jin and Wang~\cite{jin2026kdiff} independently confirm this from a dimensionality perspective. We extend these training-time observations to inference-time guidance (\cref{prop:error_amplification}).

\paragraph{Training-Free Guidance and Off-Manifold Departure.}
DPS~\cite{chung2023dps} and LGD~\cite{song2023lgd} apply gradient-based guidance for inverse and general problems. TFG~\cite{ye2024tfg} unifies prior methods~\cite{yu2023freedom,bansal2023universal} via seven hyperparameters controlling mean/variance guidance and recurrence. All methods depend on clean data estimates $\xhat$: guidance computes $\grad_{\zt}\energy(\xhat)$, so the fidelity of $\xhat$ determines both guidance quality and sample realism. Strong guidance produces degraded, off-manifold samples akin to adversarial perturbations~\cite{stutz2019disentangling,nie2022diffpure}. Theoretically, nonzero score error enables strong guidance to push samples off the data support~\cite{chidambaram2024guidance}, a failure mode confirmed for CFG~\cite{chung2024cfgpp} and extending to training-free methods. Our analysis (\cref{subsec:error_propagation}) shows that the prediction target determines the accuracy of the manifold-restoring force.

\paragraph{Guidance for Flow Matching and Evaluation.}
Feng \etal~\cite{feng2025flowguidance} derive Flow Matching guidance via velocity-field modifications; we adopt TFG's post-step correction for its unified DDPM--Flow Matching interface. Standard FID and classifier accuracy cannot distinguish on-manifold success from adversarial artifacts~\cite{shen2024tfgunderstanding,raisa2025position}: a concern borne out across 17 surveyed papers, most lacking manifold-aware metrics (\suppref{app:eval_survey}). We address this with guided-class FID (Child FID) and guidance-strength Pareto sweeps (\suppref{app:metric_justification}).

% =============================================================================
% SECTION 3: METHOD (streamlined for 14-page main body)
% =============================================================================
\section{Method}
\label{sec:method}

\subsection{Preliminaries}
\label{subsec:preliminaries}

\paragraph{Flow Matching Formulation.}
Following JiT~\cite{li2025jit}, we adopt the Flow Matching formulation~\cite{lipman2023flowmatching,liu2023rectifiedflow,albergo2023stochastic}: $\zt = t \cdot x + (1{-}t) \cdot \epsilon$ with $\epsilon \sim \N(0, \mI)$ and $t \in [0, 1]$, where $t{=}1$ corresponds to clean data and $t{=}0$ to pure noise.

\paragraph{Prediction Targets.}
Three prediction targets have been proposed:
\begin{itemize}[leftmargin=*,itemsep=2pt,parsep=0pt]
    \item $\epsilon$-prediction: The network predicts noise $\epstheta(\zt, t) \approx \epsilon$
    \item $v$-prediction: The network predicts velocity $\vtheta(\zt, t) \approx v = x - \epsilon$
    \item $x$-prediction: The network predicts clean data $\xtheta(\zt, t) \approx x$
\end{itemize}

Given each prediction, clean data can be recovered as $\xhat^{(\epsilon)} = (\zt - (1{-}t)\epstheta)/t$, $\xhat^{(v)} = \zt + (1{-}t)\vtheta$, and $\xhat^{(x)} = \xtheta$ directly. Each formula estimates the posterior mean $\E[x \mid \zt]$, the flow matching analogue of Tweedie's formula~\cite{efron2011tweedie}, but through parameterizations with fundamentally different numerical stability (\suppref{app:manifold_theory}).

\subsection{Error Propagation in Clean Data Estimation}
\label{subsec:error_propagation}

We analyze how prediction errors propagate to clean data estimates, the quantity that determines whether guided trajectories remain on the data manifold.

\begin{proposition}[Error Amplification]
\label{prop:error_amplification}
Let $\delta_\epsilon = \norm{\epsilon - \epstheta}_2$, $\delta_v = \norm{v - \vtheta}_2$, and $\delta_x = \norm{x - \xtheta}_2$ be prediction errors for each target. The error in recovered clean data is:
\begin{align}
    \norm{\xhat^{(\epsilon)} - x}_2 &= \frac{1-t}{t} \delta_\epsilon \label{eq:eps_error}\\
    \norm{\xhat^{(v)} - x}_2 &= (1-t) \delta_v \label{eq:v_error}\\
    \norm{\xhat^{(x)} - x}_2 &= \delta_x \label{eq:x_error}
\end{align}
\end{proposition}

\begin{proof}[Proof sketch]
By direct substitution of recovery formulas into the forward process. For $\epsilon$-prediction: $\xhat^{(\epsilon)} = x + \frac{1-t}{t}(\epsilon - \epstheta)$. For $v$- and $x$-prediction: analogous. Full proof in \suppref{app:proofs}.
\end{proof}

As $t \to 0$, $\epsilon$-prediction's amplification diverges while $v$- and $x$-prediction remain bounded: a strict hierarchy in how prediction errors propagate to clean-data estimates.

\paragraph{Cumulative Trajectory Divergence.}
\cref{prop:error_amplification} bounds the error at a single timestep, but guided sampling involves many steps; whether trajectories stay on the data manifold depends on how these errors accumulate. Guided sampling is iterative: errors at step $k$ corrupt the state for step $k{+}1$.

\begin{proposition}[Cumulative Guidance Error]
\label{prop:cumulative_error}
Under guided Euler sampling with $L_g$-Lipschitz guidance, the cumulative perturbation for $\epsilon$-prediction contains a $-\ln t_0$ term that diverges as $t_0 \to 0$, while the bound for $x$-prediction remains $\bigO(1{-}t_0)$. Full statement and proof in \suppref{app:proofs}.
\end{proposition}

Early high-noise steps contribute disproportionately large errors under $\epsilon$-prediction, corrupting the trajectory for all subsequent steps and driving systematic departure from the data manifold.

\begin{remark}[Manifold Force Interaction]
\label{rem:manifold_force}
The score $\grad_{\zt} \log p_t(\zt)$ decomposes into a denoising component and a \emph{manifold force}~\cite{pidstrigach2022manifold} that pulls samples toward the data manifold $\manifold$. This restoring force is weakest near $t = 0$ (pure noise) and becomes dominant only as $t \to 1$ (clean data) (\suppref{app:manifold_theory}). At the start of the reverse process, where this restoring force is weakest, $\epsilon$-prediction's $\bigO(1/t)$ error amplification (\cref{prop:error_amplification}) is simultaneously at its strongest, corrupting the clean data estimate and allowing guidance to overpower the weakened manifold force, driving samples off $\manifold$. For $x$-prediction, no such singularity exists, and guidance and manifold forces compose stably.
\end{remark}

The amplification factors above are dimension-independent, but prior work~\cite{karras2022edm,hang2023minsnr,jin2026kdiff} has shown that prediction errors themselves scale with dimension: $\delta_\epsilon \sim \sqrt{D}$ while $\delta_x \sim \sqrt{d}$ with $d \ll D$. This base-error gap compounds with the amplification hierarchy; see \suppref{app:manifold_theory} for a detailed analysis.

\subsection{Implications for Training-Free Guidance}
\label{subsec:guidance_implications}

Guidance methods compute gradients $\grad_{\zt} \energy(\xhat)$ where $\energy$ is an energy function.

\begin{theorem}[Gradient Stability]
\label{thm:gradient_stability}
For a Lipschitz energy function $\energy$ with constant $L$, the guidance gradient bound scales as $\bigO(1/t)$ for $\epsilon$-prediction, $\bigO(1)$ for $v$-prediction, and $\bigO(\norm{\jacobian_{\xtheta}})$ for $x$-prediction. Full bounds and proof in \suppref{app:proofs}.
\end{theorem}

The gradient bounds mirror the error hierarchy (assuming comparable network Jacobian norms across targets), indicating $x$-prediction yields the most stable guidance gradients among the three targets, particularly at early timesteps where global structure is determined.

\paragraph{From Error Amplification to Child FID.}
The error amplification hierarchy predicts a specific empirical signature. Cumulative trajectory perturbation (\cref{prop:cumulative_error}) means that samples departing the data manifold during early guided steps cannot re-enter the target class's natural distribution. We measure this through \emph{guided-class FID} (Child FID): FID computed between guided samples of class $y$ and real images of class $y$. A model achieving high Validity (classifier accuracy) but high C-FID produces adversarial-like successes: samples that fool the classifier without resembling real class members. The hierarchy predicts that $\epsilon$-prediction enters this adversarial regime at lower guidance strengths than $v$- or $x$-prediction, a prediction we test directly in \cref{sec:results}.

\paragraph{Applying TFG to $x$-prediction.}
\label{subsec:ma_tfg}
Under TFG~\cite{ye2024tfg}, $x$-prediction simplifies guidance: $\xhat = \xtheta(\zt, t)$ directly, bypassing the unstable recovery formula. The full algorithm and latent-space details are in \suppref{app:algorithm_latent}.

% =============================================================================
% SECTION 4: EXPERIMENTS (compressed for 14-page main body)
% =============================================================================
\section{Experiments}
\label{sec:experiments}

We evaluate whether $x$-prediction provides a better foundation for training-free guidance compared to $\epsilon$- and $v$-prediction. Full experimental protocols and additional studies are in \suppref{app:experimental_protocols}.

\subsection{Models}
\label{subsec:pretrained_models}

We use official pretrained checkpoints spanning three prediction targets and two operating spaces (\cref{tab:pretrained_models}); model sources, seeds, and the compute budget are in \suppref{app:reproducibility}. DiT-XL/2 ($\epsilon$) and SiT-XL/2 ($v$) share identical architecture, parameters (675M diffusion model + 49M VAE decoder), training data, and latent space, isolating prediction target as the sole variable. JiT-H/16 ($x$, 953M) is our primary $x$-prediction model, chosen over the larger JiT-G (2B) for parameter-scale comparability. PixelFlow ($v$, pixel) provides a critical control: against SiT it isolates operating-space effects; against JiT it isolates prediction target within pixel space.
ADM-G~\cite{dhariwal2021diffusion} (U-Net, $\epsilon$-prediction, FID 4.59) is the only available pixel-space $\epsilon$-prediction baseline (\suppref{app:confound_discussion}). JiT model variants (B/L/H/G) are detailed in \suppref{app:confound_discussion}.

\paragraph{Why models differ beyond prediction target.}
Li and He~\cite{li2025jit} showed that, under identical pixel-space transformer training, $\epsilon$-prediction achieves FID 372.38 versus 8.62 for $x$-prediction, a 43$\times$ gap indicating that $\epsilon$-prediction depends on latent-space compression to function competitively. Each model in our comparison therefore represents its prediction target's best achievable configuration; the architecture and space differences are consequences, not confounds, of prediction target choice (\suppref{app:confound_discussion}).

\begin{table}[t]
    \caption{
        \textbf{ImageNet 256$\times$256 pretrained models.}
        All use the Diffusion Transformer architecture.
        $^*$JiT-G/16 (2B) achieves FID 1.82; we use JiT-H for parameter comparability.
        $^\dagger$Includes VAE decoder (49M); under TFG the decoder is invoked at every denoising step for guidance in pixel space, making it an integral part of the generation pipeline.
        $^\ddagger$Cascade total across all stages; not directly comparable to single-pass GFLOPs.
    }
    \label{tab:pretrained_models}
    \centering
    \footnotesize
    \begin{tabular}{@{}llccccc@{}}
        \toprule
        Space & Model & Target & FID$\downarrow$ & IS$\uparrow$ & Params & GFLOPs \\
        \midrule
        \multirow{2}{*}{Pixel}
        & PixelFlow~\cite{chen2025pixelflow} & $v$ & 1.98 & 282.1 & 677M & 2909$^\ddagger$ \\
        & JiT-H/16$^*$~\cite{li2025jit} & $x$ & \textbf{1.86} & \textbf{303.4} & 953M & 182 \\
        \midrule
        \multirow{2}{*}{Latent}
        & DiT-XL/2~\cite{peebles2023dit} & $\epsilon$ & 2.27 & 278.2 & 724M$^\dagger$ & 119 \\
        & SiT-XL/2~\cite{ma2024sit} & $v$ & 2.06 & 277.5 & 724M$^\dagger$ & 119 \\
        \bottomrule
    \end{tabular}
\end{table}

\subsection{Controlled Ablation: Crossed-Lines}
\label{subsec:crossed_lines_setup}

To isolate prediction target effects from all other confounds, we train identical MLP-based flow matching models on a 2D crossed-lines dataset (two 1D line manifolds, $b{=}a$ and $b{=}{-}a$, with perpendicular Gaussian noise) embedded in ambient dimensions $D \in \{2, 8, 32, 128, 512\}$ via column-orthogonal projection (\suppref{app:experimental_protocols}).

\paragraph{Architecture and training.}
For each $D$, we train three residual MLPs (256 hidden, 5 blocks) differing only in prediction target, with flow matching for 500 epochs and identical hyperparameters; a separate MLP classifier per $D$ provides the DPS signal (details in \suppref{app:experimental_protocols}).

\paragraph{Guidance and evaluation.}
We apply DPS with guidance strength $s{=}10$, using Euler sampling with 100 steps. Each condition generates 10{,}000 samples. We report \textbf{on-manifold rate}: the fraction of generated samples within perpendicular distance $\delta$ of the true 1D manifold, where $\delta$ is the 95th percentile of ground truth perpendicular distances (measurement details in \suppref{app:experimental_protocols}). By embedding the same 1D manifold ($d{=}1$) in progressively higher ambient dimensions, this ablation directly tests the dimension-dependent error scaling analyzed in \suppref{app:manifold_theory}. Full experimental details and additional metrics are in \suppref{app:experimental_protocols}.

\subsection{Fine-Grained Bird Classification Benchmark}
\label{subsec:bird_benchmark}

We construct a hierarchical fine-grained classification benchmark to evaluate training-free guidance quality at the species level.

\paragraph{Construction.}
Starting from 30 ImageNet bird classes (e.g., \texttt{goldfinch}, \texttt{hummingbird}, \texttt{drake}), we identify fine-grained species from a 525-species bird classification dataset~\cite{piosenka2023birds} (previously used for fine-grained guidance evaluation by Ye~\etal~\cite{ye2024tfg}) that map to each parent, yielding 143 species nested within 30 parent classes (2--20 species per parent, mean 4.8; dataset details in \suppref{app:experimental_protocols}). This hierarchy naturally separates two levels of conditioning: classifier-free guidance (CFG)~\cite{ho2022cfg} steers toward the parent class using the model's own class conditioning, while gradient-based guidance (DPS~\cite{chung2023dps}) steers each sample toward a specific species via an external fine-grained classifier\footnote{Guidance: \url{https://huggingface.co/dennisjooo/Birds-Classifier-EfficientNetB2}; evaluation: \url{https://huggingface.co/chriamue/bird-species-classifier}}. We use separate classifiers for guidance and evaluation to avoid circular evaluation~\cite{shen2024tfgunderstanding} (\suppref{app:metric_justification}).

\paragraph{Why fine-grained birds?}
Bird classes are already part of ImageNet's label space, so pretrained models can generate them without domain transfer. Species differ in subtle plumage, beak, and eye markings, requiring semantic shifts that make the gap between classifier-fooling artifacts and on-manifold guidance more visible. The hierarchical structure (parent class $\to$ species) naturally separates the roles of CFG and DPS, enabling controlled guidance-strength sweeps.

\paragraph{Scale.}
Each condition generates 64 samples per species across all 143 classes (9{,}152 images), with $\rho$ swept across 5--10 values per model for six models in total.

\subsection{Evaluation Protocol}
\label{subsec:eval_protocol}

Standard evaluation reports Validity and FID at a fixed guidance strength, but neither metric detects manifold departure. Validity rewards any image the classifier labels correctly, including off-manifold artifacts that fool the network. Parent FID (P-FID, against the full ImageNet reference) rises whether guidance degrades images or successfully shifts them toward a sub-class. A survey of 17 method papers reveals that manifold-aware metrics and guidance-strength sweeps remain uncommon: only two report manifold-aware metrics (\suppref{app:eval_survey}).

\paragraph{Child FID and guidance sweeps.}
We propose \textbf{Child FID} (C-FID): FID~\cite{heusel2017fid} between guided samples and the \emph{target domain}, the bird species dataset (justified in \suppref{app:metric_justification}). Rising P-FID paired with falling C-FID signals successful guidance; both rising signals degradation. Rather than single-point comparisons, we sweep $\rho$ and plot Pareto curves (P-FID vs.\ Validity, P-FID vs.\ C-FID; \cref{fig:rho_sweep}).

\paragraph{Inference setup.}
We standardize all models to NFE$\approx$100 using each model's native sampler: DiT uses 100-step DDPM, SiT and JiT use 50-step Heun, and PixelFlow uses 30-step$\times$4-stage Euler (NFE$=$120). Latent models use the ema VAE decoder. These settings reduce NFE from each model's published optimum to enable fair comparison; full configurations are in \suppref{app:experimental_protocols} (\cref{tab:inference_hyperparams}).

\paragraph{Guidance.}
We apply DPS~\cite{chung2023dps}, adding $\rho \nabla_{z_t} \log p(y \mid \xhat)$ at each denoising step with $x$-space corrections. Latent models (DiT, SiT) require VAE decoder passes at each guidance step for pixel-space gradients; this overhead is absent for pixel-space models (JiT, PixelFlow). Full configuration details are in \suppref{app:experimental_protocols}.

% =============================================================================
% SECTION 5: RESULTS
% =============================================================================
\FloatBarrier
\section{Results}
\label{sec:results}

\subsection{Crossed-Lines Ablation}
\label{subsec:crossed_lines_results}

The crossed-lines toy experiment (\cref{fig:teaser}b, \cref{fig:crossed_lines_compact}) isolates the effect of prediction target in a controlled setting where architecture and training are identical. As ambient dimension increases from $D{=}2$ to $D{=}512$, the hierarchy predicted by \cref{prop:error_amplification} emerges (\cref{fig:onmanifold_vs_dim}): $x$-prediction maintains 93.3\% on-manifold rate at $D{=}512$, $v$-prediction degrades to 21.5\%, and $\epsilon$-prediction collapses to 0.5\%, consistent with $\sqrt{D}$ error scaling (\suppref{app:manifold_theory}). Full metrics and the complete visualization grid are in \suppref{app:experimental_protocols} (\cref{tab:crossed_lines_full,fig:crossed_lines_grid}).

\FloatBarrier
\subsection{Fine-Grained Bird Classification}
\label{subsec:bird_results}

\begin{figure}[t]
    \centering
    \begin{subfigure}[b]{0.49\textwidth}
        \includegraphics[width=\textwidth]{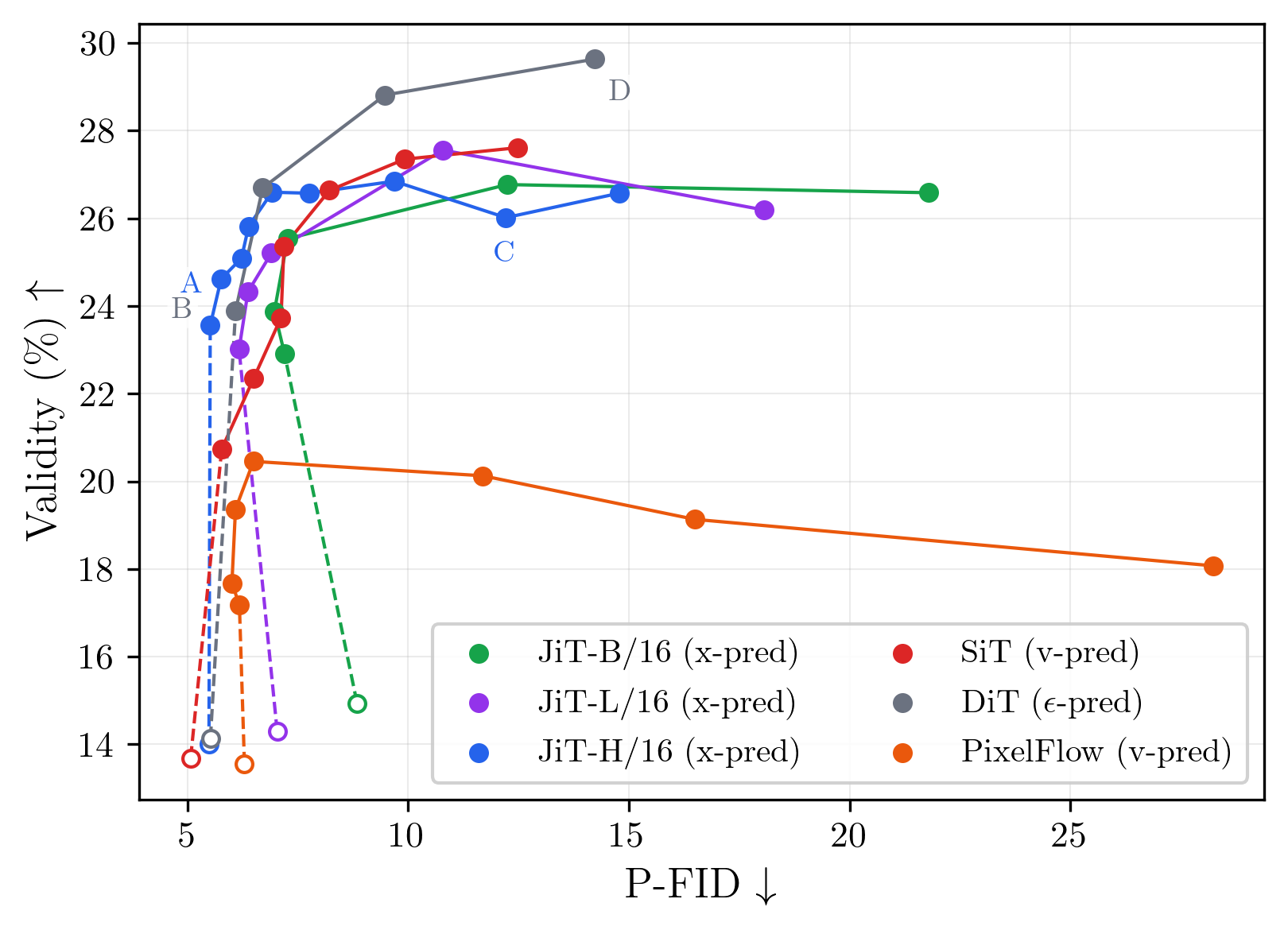}
        \caption{P-FID vs.\ Validity}
        \label{fig:rho_fid_validity}
    \end{subfigure}
    \hfill
    \begin{subfigure}[b]{0.49\textwidth}
        \includegraphics[width=\textwidth]{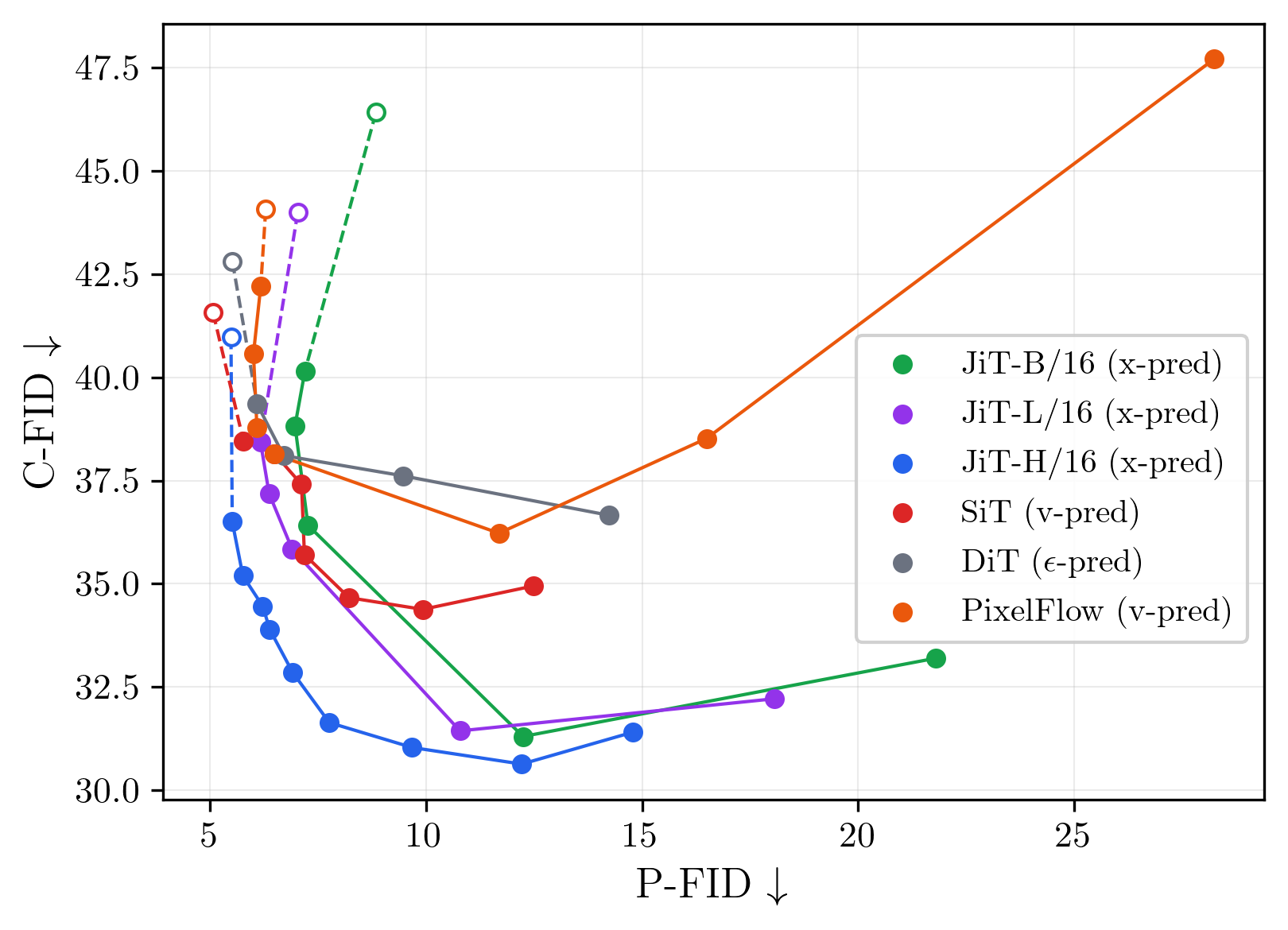}
        \caption{P-FID vs.\ C-FID}
        \label{fig:rho_fid_cfid}
    \end{subfigure}
    \caption{
        \textbf{Guidance-quality Pareto frontiers on fine-grained bird classification.}
        Each curve traces one model across guidance strengths $\rho$; open markers denote CFG-only baselines ($\rho{=}0$), dashed segments connect to the first guided setting. Each data point represents 9{,}152 generated images (143 species $\times$ 64 samples).
        \textbf{(a)}~P-FID vs.\ Validity. Labels A--D mark the operating points visualized in \cref{fig:qualitative}.
        \textbf{(b)}~P-FID vs.\ C-FID.
    }
    \label{fig:rho_sweep}
\end{figure}

\cref{fig:rho_sweep} presents guidance-strength sweeps across six models spanning three prediction targets and two operating spaces.

\paragraph{Validity alone is misleading.}
All models show increasing Validity with $\rho$ (\cref{fig:rho_fid_validity}), but the quality cost differs substantially across targets. At matched Validity ($\approx$26.6\%), JiT-H ($\rho{=}3$, P-FID 6.9) and DiT ($\rho{=}0.1$, P-FID 6.7) appear equivalent. Under stronger guidance, DiT reaches 29.6\% Validity but at P-FID 14.2 ($2.6{\times}$ its baseline), while JiT maintains low P-FID throughout. Single-point comparisons obscure this divergence.

\paragraph{Child FID reveals manifold fidelity.}
\cref{fig:rho_fid_cfid} disambiguates the P-FID trade-off. At matched Validity ($\approx$26.6\%), JiT-H achieves C-FID 32.9 versus DiT's 38.1 and SiT's 34.7, a 5.2-point gap between $x$- and $\epsilon$-prediction at identical classifier confidence. DiT's trajectory is revealing: from $\rho{=}0.1$ to $\rho{=}0.5$, Validity gains come with stagnant C-FID and collapsing P-FID, a characteristic signature of adversarial-like guidance predicted by \cref{prop:error_amplification}. Qualitative inspection confirms this pattern: DiT's guided samples concentrate on a narrow set of visual templates, whereas JiT produces diverse compositions across the same species (\cref{fig:qualitative}).

\paragraph{Mode collapse.}
We quantify the diversity loss with DINOv2 Precision and Recall~\cite{naeem2020reliable} (\cref{fig:prdc_main}). Mode collapse corresponds to high Precision with low Recall: a model that produces a narrow but realistic subset covers the target distribution poorly. DiT ($\epsilon$) follows this pattern, with Precision peaking at 0.24 while Recall stays around 0.49, whereas JiT-H ($x$) has lower Precision but higher Recall (up to 0.59). This accounts for the high Precision of $\epsilon$-prediction: Precision measures only the realism of generated samples, not coverage of the target distribution, and Recall shows that DiT covers less of it. The behavior is consistent with the error amplification hierarchy (\cref{prop:error_amplification}): under guidance, $\epsilon$-prediction concentrates samples on a narrow set of classifier-activating features and reduces the support that $x$-prediction retains.

\begin{figure}[t]
    \centering
    \begin{subfigure}[b]{0.49\textwidth}
        \includegraphics[width=\textwidth]{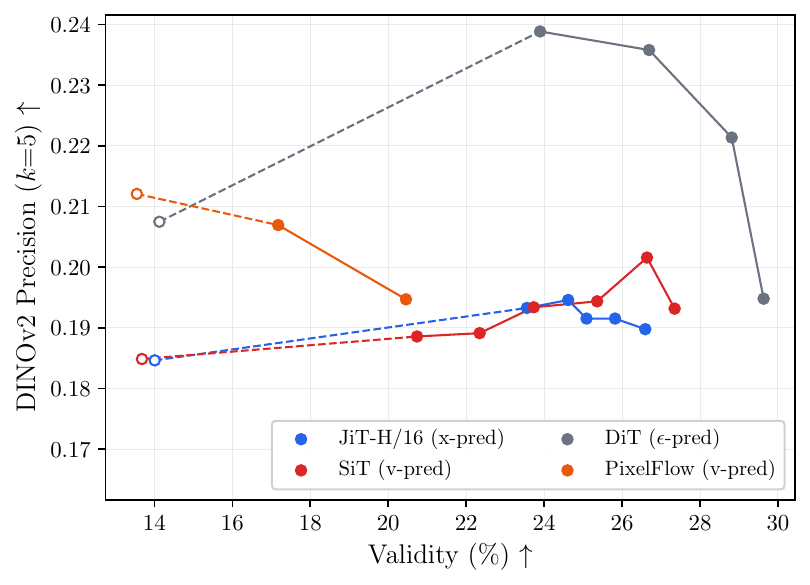}
        \caption{Validity vs.\ Precision}
        \label{fig:prdc_precision}
    \end{subfigure}
    \hfill
    \begin{subfigure}[b]{0.49\textwidth}
        \includegraphics[width=\textwidth]{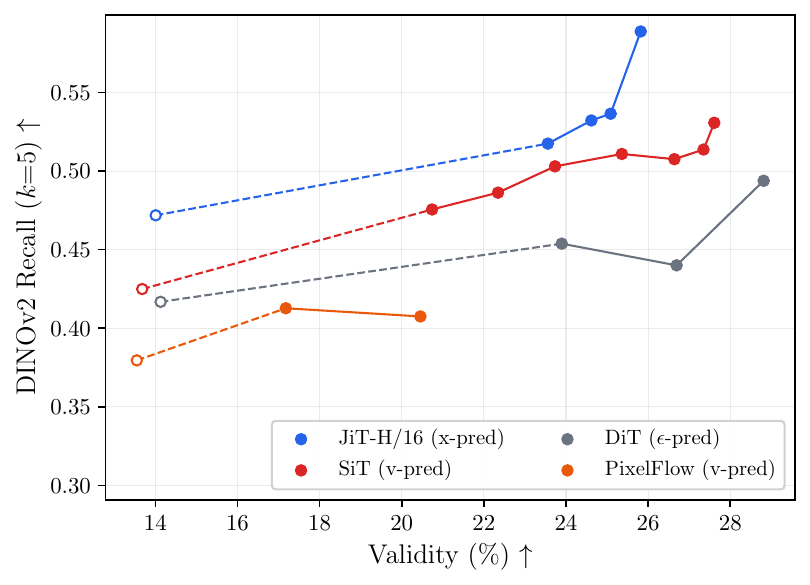}
        \caption{Validity vs.\ Recall}
        \label{fig:prdc_recall}
    \end{subfigure}
    \caption{
        \textbf{Precision and Recall under guidance on fine-grained bird classification.}
        DINOv2 $k$-NN Precision and Recall traced over the DPS $\rho$-sweep; open markers denote CFG-only baselines ($\rho{=}0$), dashed segments connect to the first guided setting.
        \textbf{(a)}~$\epsilon$-prediction (DiT) attains the highest Precision (fidelity), while
        \textbf{(b)}~$x$-prediction (JiT) attains the highest Recall (coverage). The joint pattern of high Precision and low Recall is the mode-collapse signature of $\epsilon$-prediction. Full six-model curves are in \cref{fig:prdc_sweep} (\suppref{app:prdc}).
    }
    \label{fig:prdc_main}
\end{figure}

\paragraph{Prediction target vs.\ operating space.}
PixelFlow ($v$-prediction, pixel space) provides a critical control for JiT ($x$-prediction, pixel space). Despite sharing the same operating space, PixelFlow exhibits a substantially worse Pareto frontier: its C-FID initially improves from 44.1 ($\rho{=}0$) to 36.2 ($\rho{=}2$) but then \emph{increases} to 47.7 at strong guidance (\cref{fig:rho_fid_cfid}), signaling manifold departure. In contrast, JiT-H's C-FID continues decreasing throughout the sweep, reaching 30.6 at $\rho{=}8$. This indicates that the prediction target, rather than the operating space, is the primary determinant of guidance robustness in this comparison.

\paragraph{Latent space models.}
DiT and SiT operate in a 32$\times$32 VAE latent space, requiring decoder passes at each guidance step. SiT ($v$-prediction) dominates DiT ($\epsilon$-prediction), consistent with bounded error attenuation (\cref{prop:error_amplification}), but both are dominated by JiT in C-FID. The VAE's 8$\times$ downsampling may further limit fine-grained guidance resolution, though we do not isolate this factor.

\paragraph{Scaling with model capacity.}
Among JiT variants (B/L/H), larger models achieve strictly better Pareto frontiers: JiT-H reaches C-FID 30.6 versus JiT-B's 31.3 at comparable P-FID. Increased capacity yields higher-fidelity $\xhat$ estimates and more accurate guidance gradients, a \emph{guidance scaling} effect.

\paragraph{Connection to theory.}
The C-FID evidence supports the error amplification hierarchy (\cref{prop:error_amplification,thm:gradient_stability,prop:cumulative_error}): samples departing $\manifold$ at early steps cannot return to realistic distributions, inflating C-FID even when the classifier is fooled. $x$-prediction tolerates aggressive guidance with less manifold degradation. The same ordering holds under other gradient-based methods (LGD~\cite{song2023lgd}, FreeDoM~\cite{yu2023freedom}) and on a second fine-grained domain, a 34-species butterfly benchmark (\suppref{app:experimental_protocols}, \cref{fig:tfg_family,fig:butterfly}): guidance quality follows the prediction target, not the specific method or domain. Additional ablations are in \suppref{app:experimental_protocols}.

\FloatBarrier
\subsection{Style Transfer}
\label{subsec:style_results}

To test whether the prediction target hierarchy extends beyond classification, we apply DPS to style transfer: guiding generation toward a target visual style via CLIP Gram matrix matching~\cite{gatys2016styletransfer}. We use CLIP ViT-B/16~\cite{radford2021clip} for guidance and evaluate with Gram Distance (CLIP ViT-B/32; lower = stronger style match) and Content Accuracy (DeiT-Small~\cite{touvron2021deit} top-1 accuracy on the original ImageNet class; higher = better content preservation). Four WikiArt\footnote{\url{https://www.wikiart.org/}} styles are evaluated across 100 ImageNet classes (400 images per model/$\rho$).

\begin{wrapfigure}{r}{0.5\textwidth}
    \centering
    \vspace{-12pt}
    \includegraphics[width=0.48\textwidth]{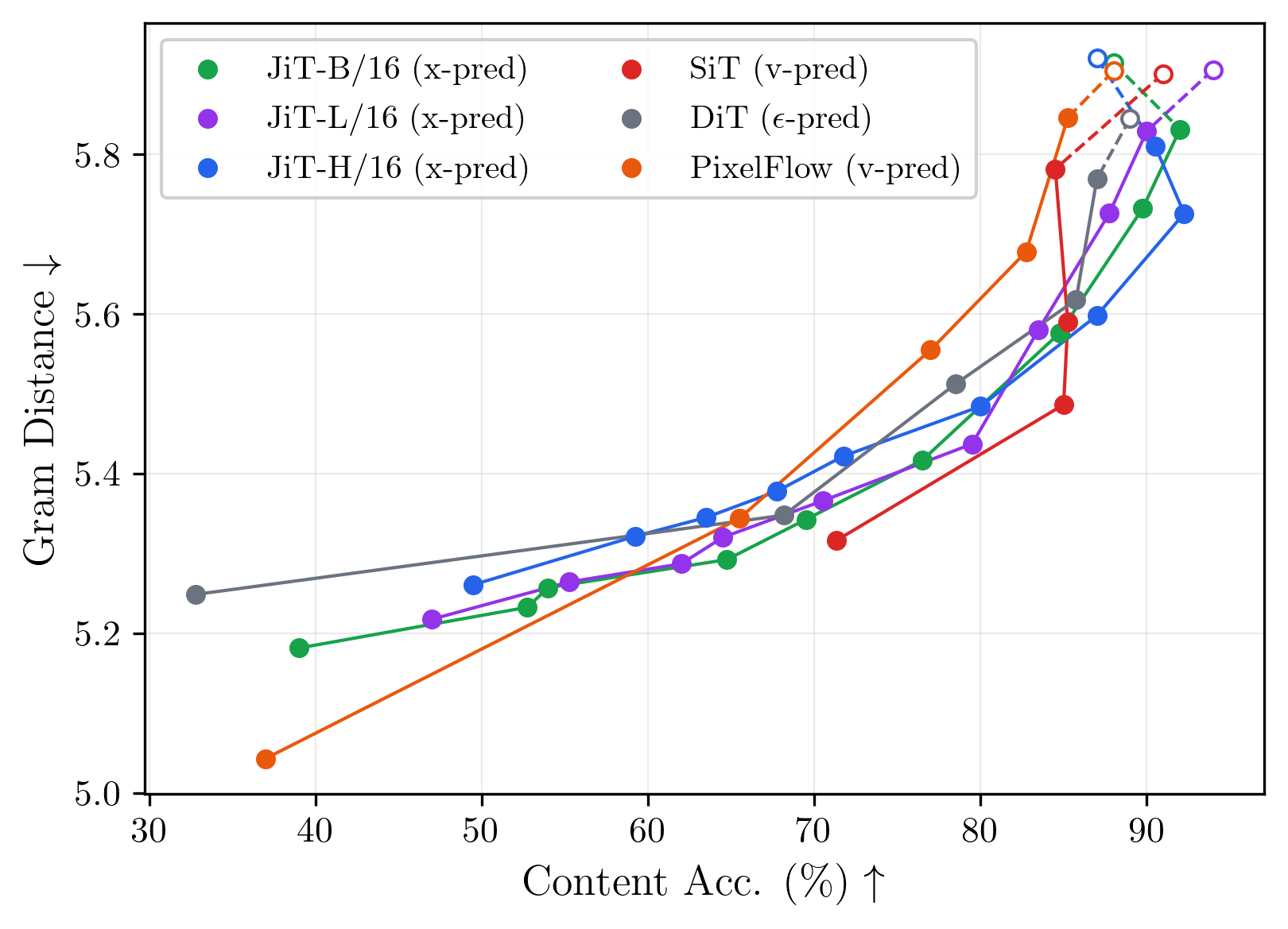}
    \caption{
        \textbf{Style transfer: Gram Distance vs.\ Content Accuracy.}
        Curves trace models as $\rho$ increases (right to left). Open markers = CFG-only baselines. Lower-right is preferred.
    }
    \label{fig:style_sweep}
    \vspace{-10pt}
\end{wrapfigure}

\cref{fig:style_sweep} presents Gram Distance vs.\ Content Accuracy Pareto frontiers. All models trade content preservation for style fidelity as $\rho$ increases, but the degradation rate varies considerably across prediction targets. Overall, the differences between models are less pronounced than in fine-grained classification (\cref{subsec:bird_results}).

\paragraph{$\epsilon$- and $v$-prediction collapse under strong style guidance.}
DiT's Content Accuracy drops from 89\% ($\rho{=}0$) to 1.5\% ($\rho{=}10$), effectively random, while achieving Gram Distance 5.30. The low Gram Distance is meaningless when images no longer depict recognizable content. PixelFlow ($v$-prediction, pixel space) follows a similar pattern: at $\rho{=}4$, it achieves the lowest Gram Distance of any model (5.04) but at only 37\% Content Accuracy, where images lose semantic coherence.

\paragraph{$x$-prediction preserves content over a wider guidance range.}
At high Content Accuracy ($>$80\%), JiT-H achieves better Gram Distance than DiT and PixelFlow: at comparable Content Accuracy ($\approx$80\%), JiT-H ($\rho{=}10$, Gram Distance 5.48) outperforms DiT ($\rho{=}0.5$, Gram Distance 5.62) and PixelFlow ($\rho{=}1$, Gram Distance 5.56). JiT-H further reaches Content Accuracy 49.5\% at $\rho{=}50$ (Gram Distance 5.26), matching DiT's best Gram Distance while retaining meaningful content fidelity.

\paragraph{SiT is competitive at moderate guidance.}
SiT ($v$-prediction, latent) achieves a competitive Pareto frontier at moderate $\rho$: Gram Distance 5.49 at Content Accuracy 85.0\% ($\rho{=}1$). At stronger guidance ($\rho{=}4$), SiT already drops to 71\% Content Accuracy while JiT-H retains 87\%; by $\rho{=}10$, SiT falls to 38\% while JiT-H retains 80\%. The Pareto frontiers of SiT and JiT overlap in the moderate-guidance regime and diverge at the extremes.

\paragraph{Consistent failure modes across tasks.}
Although the quantitative gap between models is smaller than in fine-grained classification, the qualitative failure modes are consistent: DiT's guided samples exhibit mode collapse and loss of background detail, while JiT-H preserves compositional diversity (\cref{fig:style_vis} in \suppref{app:experimental_protocols}). This suggests that the error amplification hierarchy (\cref{prop:error_amplification}) manifests across guidance tasks, even when the aggregate metrics show smaller differences. The Gram matrix guidance signal captures aggregate texture statistics rather than fine-grained spatial details, which may explain the reduced quantitative separation.

\paragraph{Additional experiments.}
We also evaluate DPS on two inverse problems (Gaussian deblurring and 4$\times$ super-resolution), where $x$-prediction again achieves the best perceptual quality (LPIPS) across all models. Full results and discussion are in \suppref{app:experimental_protocols} (\cref{app:inverse_problems}).

\FloatBarrier
\subsection{Qualitative Analysis}
\label{subsec:qualitative}

\begin{figure}[p]
    \centering
    % (a) JiT-H moderate guidance
    \begin{subfigure}[b]{\textwidth}
        \includegraphics[width=\textwidth]{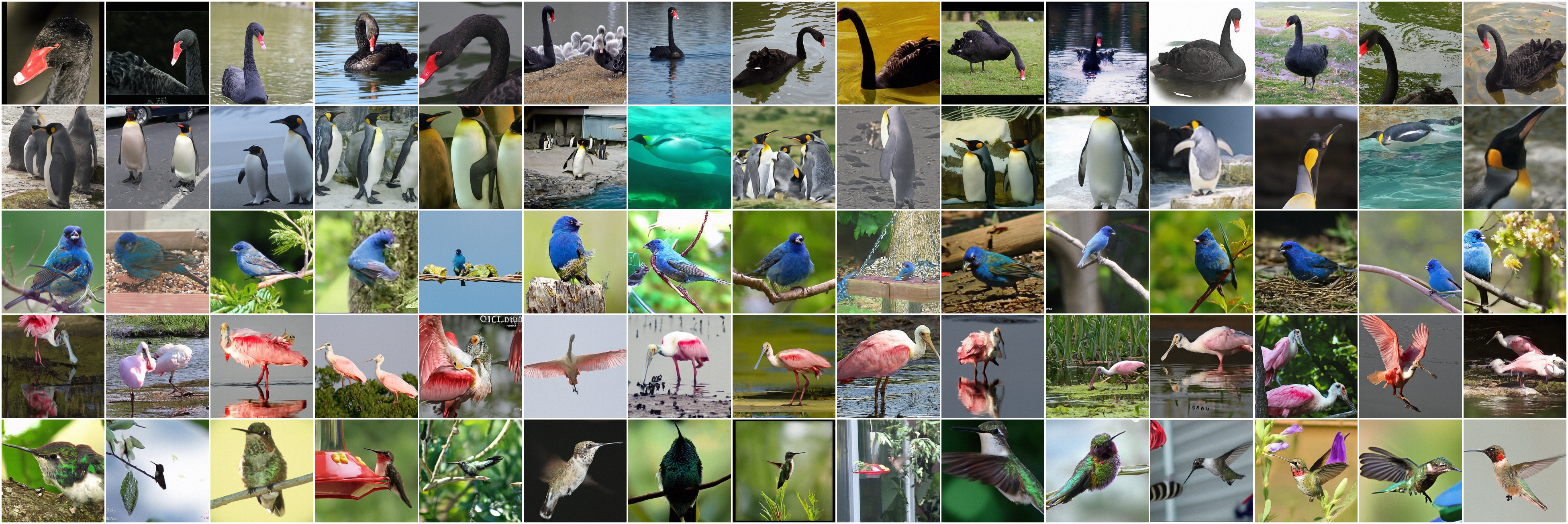}
        \caption{JiT-H ($x$, $\rho{=}1.5$): P-FID 6.2, C-FID 34.5, Val.\ 25.1\% (point A in \cref{fig:rho_fid_validity})}
        \label{fig:vis_on_jit}
    \end{subfigure}

    \vspace{1pt}

    % (b) DiT moderate guidance
    \begin{subfigure}[b]{\textwidth}
        \includegraphics[width=\textwidth]{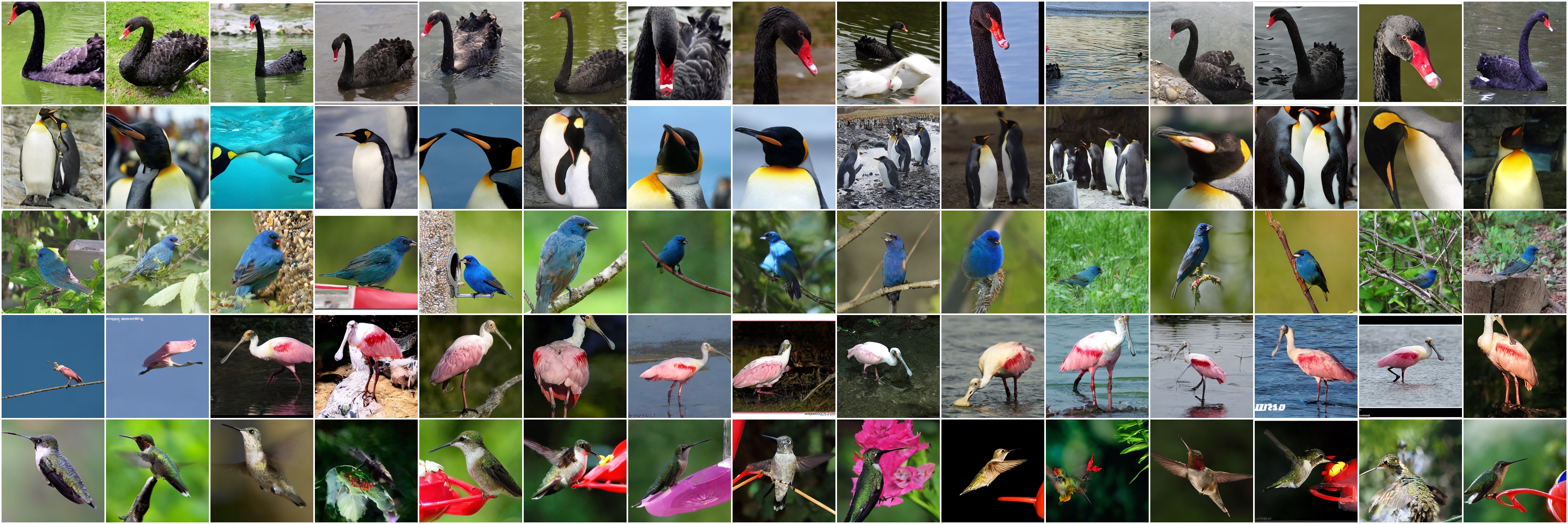}
        \caption{DiT ($\epsilon$, $\rho{=}0.05$): P-FID 6.1, C-FID 39.4, Val.\ 23.9\% (point B in \cref{fig:rho_fid_validity})}
        \label{fig:vis_on_dit}
    \end{subfigure}

    \vspace{1pt}

    % (c) JiT-H strong guidance
    \begin{subfigure}[b]{\textwidth}
        \includegraphics[width=\textwidth]{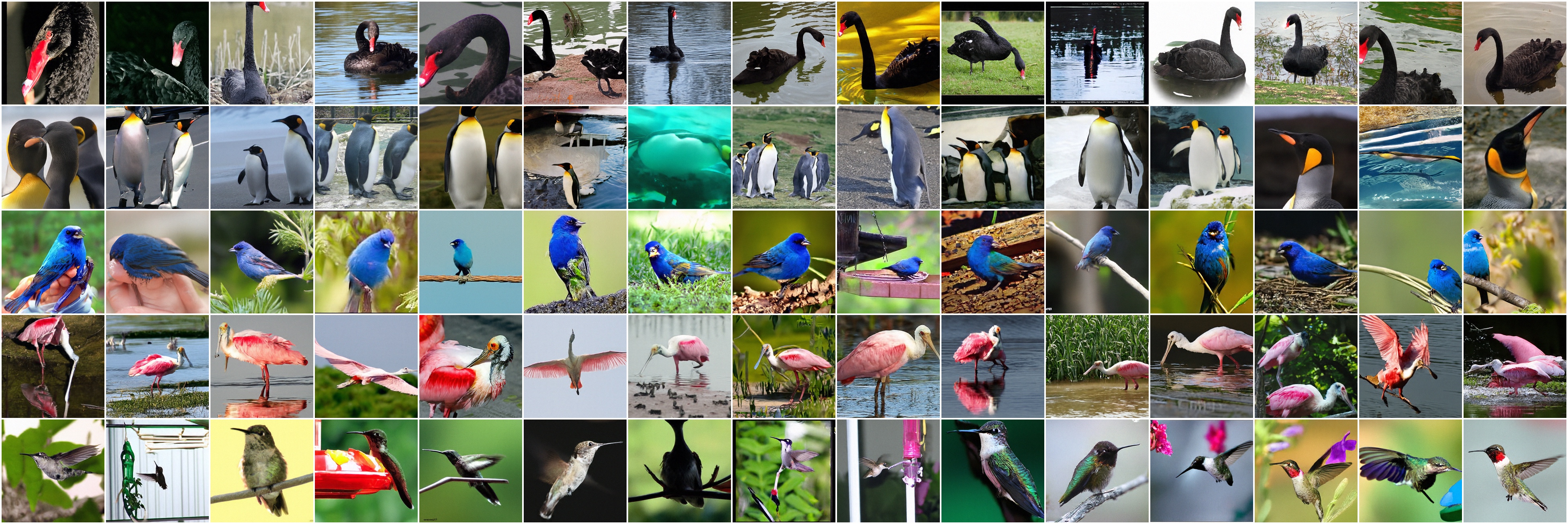}
        \caption{JiT-H ($x$, $\rho{=}8$): P-FID 12.2, C-FID 30.6, Val.\ 26.0\% (point C in \cref{fig:rho_fid_validity})}
        \label{fig:vis_off_jit}
    \end{subfigure}

    \vspace{1pt}

    % (d) DiT strong guidance
    \begin{subfigure}[b]{\textwidth}
        \includegraphics[width=\textwidth]{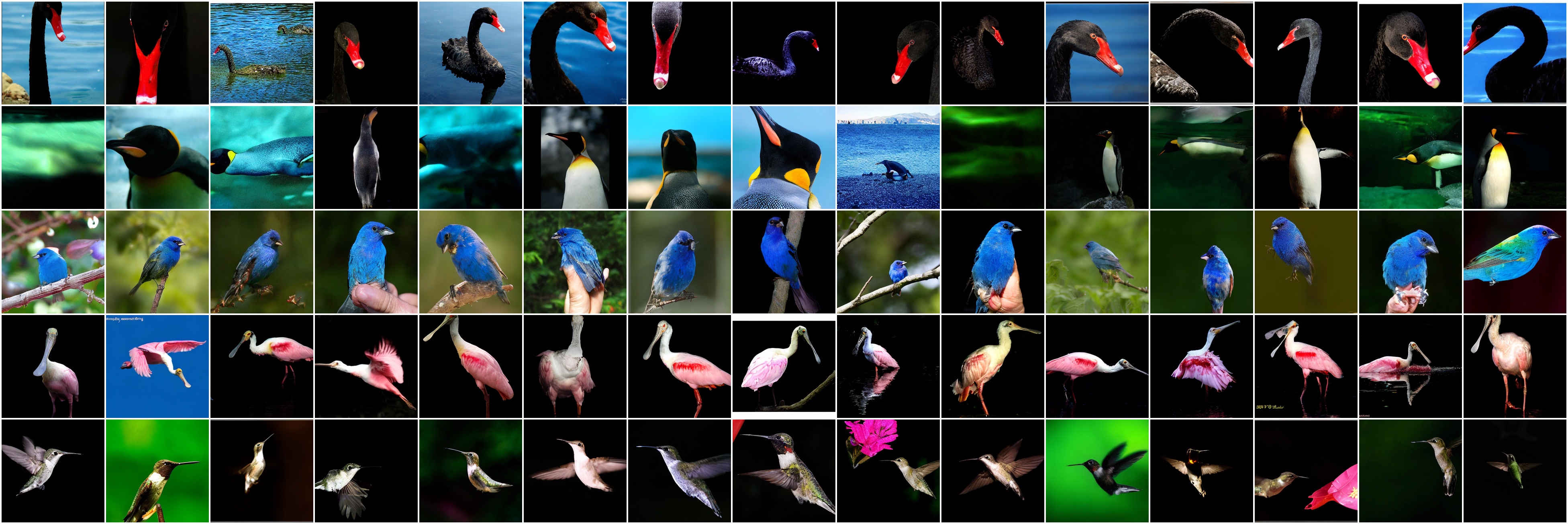}
        \caption{DiT ($\epsilon$, $\rho{=}0.5$): P-FID 14.2, C-FID 36.7, Val.\ 29.6\% (point D in \cref{fig:rho_fid_validity})}
        \label{fig:vis_off_dit}
    \end{subfigure}

    \caption{
        \textbf{Guided generation on five bird species (15 random samples each, no cherry-picking).}
        (a),(b): moderate guidance (matched P-FID${\approx}$6); (c),(d): strong guidance. Rows: Black Swan, Emperor Penguin, Painted Bunting, Roseate Spoonbill, Ruby-throated Hummingbird.
    }
    \label{fig:qualitative}
\end{figure}

\cref{fig:qualitative} presents randomly drawn (not cherry-picked) guided samples from JiT-H ($x$-prediction) and DiT ($\epsilon$-prediction) at two guidance regimes: moderate ($\rho$ chosen for matched P-FID${\approx}$6; (a),(b)) and strong ((c),(d)), across five visually distinct bird species. Corresponding visualizations for SiT and PixelFlow are in \suppref{app:experimental_protocols} (\cref{fig:qualitative_appendix}).

\paragraph{Off-manifold degradation is visible but not catastrophic for $x$-prediction.}
Under strong guidance, JiT-H (\cref{fig:vis_off_jit}) shows visible quality degradation (posterization and reduced fine detail) but maintains diverse poses, varied backgrounds, and recognizable species-specific features. DiT (\cref{fig:vis_off_dit}) achieves 29.6\% Validity (higher than JiT-H's 26.0\%), yet its samples exhibit pronounced mode collapse: many images share similar poses, framing, and uniform dark backgrounds, particularly visible in Black Swan and Painted Bunting.

\paragraph{Mode concentration as a failure signature.}
The samples make concrete the mode collapse quantified in \cref{subsec:bird_results} (\cref{fig:prdc_main}): the low Recall measured there corresponds to the narrow set of poses and backgrounds DiT repeats here. C-FID registers the same effect: DiT's C-FID (36.7) remains worse than JiT-H's (30.6) despite higher Validity. The samples show how it appears: DiT's $\epsilon$-prediction reuses a small set of classifier-friendly patterns rather than covering each species' variation, whereas JiT-H retains diverse poses and compositions.

\paragraph{Practical implication.}
Low Validity with preserved image quality (JiT-H at strong guidance) is preferable to high Validity with degraded diversity (DiT): a user can retry generation for the correct class, but cannot recover from mode collapse or manifold departure. This asymmetry further motivates C-FID over Validity as the primary evaluation metric for training-free guidance.

% =============================================================================
% SECTION 6: CONCLUSION
% =============================================================================
\section{Conclusion}
\label{sec:conclusion}

Whether training-free guidance fails gracefully (missing the target but producing a realistic image) or catastrophically (collapsing into off-manifold artifacts) depends on the prediction target. The mechanism is the recovery formula: $\epsilon$-prediction divides by~$t$, amplifying errors unboundedly at high noise; $v$-prediction incurs bounded amplification; $x$-prediction incurs none. On ImageNet, C-FID supports this hierarchy with a 5.2-point gap at matched Validity (manifold damage invisible to standard evaluation), and PixelFlow's C-FID reversal under strong guidance isolates prediction target as the decisive factor. The hierarchy extends to style transfer, establishing $x$-prediction as the target that keeps guidance failures graceful.

The fine-grained bird benchmark and C-FID protocol we introduce let any ImageNet-scale diffusion model be immediately tested for manifold-aware guidance quality, and can serve as a practical diagnostic for evaluating new models.

\paragraph{Limitations.}
Our comparison involves models differing in architecture and operating space beyond prediction target; no single comparison perfectly isolates it, and the conclusion instead rests on the convergence of five independent control experiments (\suppref{app:confound_discussion}), of which the latent-space DiT vs.\ SiT pair and the pixel-space PixelFlow vs.\ JiT pair are the most controlled within each space. The theoretical analysis assumes Lipschitz energy functions and well-trained models. We evaluate gradient-based TFG methods (DPS, LGD, FreeDoM), not the full TFG parameter space, and all experiments use 256${\times}$256 resolution; scaling behavior at higher resolutions remains to be tested.

\paragraph{Future Work.}
Inference-time scaling methods~\cite{ma2025inferencescaling,kim2025das}, which are search-based and SMC approaches that compound guidance across many candidates, depend critically on $\xhat$ quality; the prediction target hierarchy should govern their sample efficiency. Video and text-to-image generation present higher-dimensional sequential settings where error amplification compounds across frames, and our dimension scaling analysis predicts even larger gaps between prediction targets.

Not all prediction targets keep guided samples on the manifold, but $x$-prediction does. For practitioners building inference-time control pipelines, it is the robust foundation.

% Acknowledgements (camera-ready; remove for any anonymous resubmission)
\subsubsection*{Acknowledgements}
This research was supported by Seoul National University of Science and Technology.

% Flush all pending floats before references/appendix
\clearpage

% =============================================================================
% REFERENCES
%   Main-paper (MAINONLY) build prints them here. The full build defers them to
%   after the appendix (below) so the split-off supplementary carries its own
%   reference list.
% =============================================================================
\bibliography{ref}
\bibliographystyle{splncs04}

% =============================================================================
% APPENDIX (full build only; split out as the standalone supplementary).
%   Excluded from the MAINONLY build, which produces the appendix-free
%   source/PDF for Springer. References from the main body into the appendix are
%   resolved via xr-hyper against the full build's root.aux (see preamble).
% =============================================================================
\ifdefined\MAINONLY\else
\newpage
\appendix

\phantomsection
\label{supp:start}
\section*{Supplementary Material: Table of Contents}
\begin{itemize}[leftmargin=1.5em,itemsep=1pt,parsep=0pt]
    \item \textbf{Appendix A} -- Manifold Hypothesis and Score Theory
    \item \textbf{Appendix B} -- Full Proofs
    \item \textbf{Appendix C} -- Evaluation Metric Justification
    \item \textbf{Appendix D} -- Evaluation Practices in Training-Free Guidance Literature
    \item \textbf{Appendix E} -- Models and the Architecture-Confound Discussion
    \item \textbf{Appendix F} -- DPS Algorithm and Latent-Space Details
    \item \textbf{Appendix G} -- Experimental Protocols and Full Results
    \item \textbf{Appendix H} -- Precision--Recall Analysis
    \item \textbf{Appendix I} -- Reproducibility Statement
\end{itemize}

% =============================================================================
% APPENDIX B: MANIFOLD HYPOTHESIS AND SCORE THEORY
% =============================================================================
\section{Manifold Hypothesis and Score Theory}
\label{app:manifold_theory}

\begin{definition}[Manifold Hypothesis]
\label{def:manifold}
Natural data $x$ lies on a compact smooth submanifold $\manifold \subset \R^D$ with intrinsic dimension $d \ll D$~\cite{fefferman2016manifold,narayanan2010manifold,farghly2025manifoldhypothesis}. Noise $\epsilon \sim \N(0, \mI)$ is distributed across the full ambient space $\R^D$.
\end{definition}

\subsection{Score Decomposition and Manifold Force}
\label{app:score_decomposition}

\begin{remark}[Score Decomposition]
\label{rem:score_decomposition}
For the JiT forward process $\zt = t \cdot x + (1-t) \cdot \epsilon$ with $\epsilon \sim \N(0, \mI)$, the conditional density is $p(\zt \mid x) = \N(tx, (1-t)^2 \mI)$. By Fisher's identity, the marginal score admits the exact decomposition:
\begin{equation}
    \grad_{\zt} \log p_t(\zt) = \frac{1}{(1-t)^2} \Big( t \cdot \E[x \mid \zt] - \zt \Big)
    \label{eq:score_decomposition}
\end{equation}
Equivalently, $\E[x \mid \zt] = \frac{1}{t} \big( \zt + (1-t)^2 \grad_{\zt} \log p_t(\zt) \big)$. This is the JiT analogue of Pidstrigach's manifold-tangential decomposition~\cite{pidstrigach2022manifold}.
\end{remark}

\paragraph{Derivation.}
The conditional density $p(\zt \mid x) = \N(\zt; tx, (1-t)^2 \mI)$ gives $\grad_{\zt} \log p(\zt \mid x) = -\frac{1}{(1-t)^2}(\zt - tx)$. By Fisher's identity:
\begin{equation}
    \grad_{\zt} \log p_t(\zt) = \E\big[\grad_{\zt} \log p(\zt \mid X) \mid \zt\big] = \frac{1}{(1-t)^2} \big(t \cdot \E[X \mid \zt] - \zt\big)
\end{equation}
This recovers Tweedie's formula~\cite{robbins1956empirical,efron2011tweedie,kim2021noise2score} in the flow matching setting.

\begin{remark}[Manifold Force]
\label{rem:manifold_force_full}
The factor $1/(1-t)^2$ in \cref{eq:score_decomposition} is the score's manifold-restoring component~\cite{pidstrigach2022manifold}, with strength $\Theta(\sigma^{-2})$ ($\sigma = 1{-}t$)~\cite{li2025scoresgeometry}. It diverges as $t \to 1$ (strong pull toward $\manifold$ near clean data) and equals ${\approx}1$ near $t = 0$ (weak pull at high noise). Additionally, the score's normal component scales as $\bigO(1/\sigma)$ while the tangential component remains $\bigO(1)$~\cite{liu2025scoresingularity}, compounding $\epsilon$-prediction's estimation difficulty at high noise.
\end{remark}

\subsection{Score Error from Conditional Mean Error}
\label{app:score_error}

\begin{proposition}[Score Error]
\label{prop:score_error}
Let $\xhat(\zt, t) = \E[x \mid \zt]$ and define $s^\star(\zt, t) = \grad_{\zt} \log p_t(\zt)$. For any approximation $\tilde{x}(\zt, t)$, define the induced score estimate $\tilde{s}(\zt, t) = \frac{1}{(1-t)^2}(t \cdot \tilde{x}(\zt, t) - \zt)$. Then:
\begin{equation}
    \norm{\tilde{s}(\zt, t) - s^\star(\zt, t)}_2 = \frac{t}{(1-t)^2} \norm{\tilde{x}(\zt, t) - \xhat(\zt, t)}_2
    \label{eq:score_error_bound}
\end{equation}
\end{proposition}

Proof in \suppref{app:proofs}.

\subsection{Dimension Scaling of Prediction Errors}
\label{app:dimension_scaling}

\begin{remark}[Dimension Scaling]
\label{rem:dimension_scaling}
The amplification factors in \cref{prop:error_amplification} are dimension-independent, yet empirically the gap widens with ambient dimension $D$. Under approximately isotropic residual errors, concentration of measure~\cite{vershynin2018high} gives $\delta_\epsilon \sim \sqrt{D}$ ($\epsilon$-prediction resolves all $D$ noise components), $\delta_x \sim \sqrt{d}$ ($x$-prediction maps to $\manifold$ of dimension $d$), confirmed empirically by Li and He~\cite{li2025jit}. For $v$-prediction, $v = x - \epsilon$ mixes a $d$-dimensional manifold component and a $D$-dimensional ambient component, yielding $\delta_v \sim \sqrt{d + D_{\mathrm{eff}}}$ where $d \leq D_{\mathrm{eff}} \leq D$ depends on how many noise dimensions the network resolves. Thus $\delta_x \ll \delta_v \leq \delta_\epsilon$, establishing a strict hierarchy in base prediction error.

For ImageNet ($d \approx 26$--$43$~\cite{pope2021intrinsic}, $D = 196{,}608$), this yields a $68$--$87\times$ gap between $\delta_\epsilon$ and $\delta_x$ in base prediction error before amplification, a distinct quantity from the $43\times$ FID gap of \cref{subsec:pretrained_models} (which measures end-to-end generation quality, not raw prediction error). The $(1{-}t)$ attenuation in $v$-prediction's recovery formula partially compensates for $\delta_v > \delta_x$, but at $t \approx 0$ (where guidance matters most) this attenuation vanishes, leaving the base error hierarchy exposed.
\end{remark}

% =============================================================================
% APPENDIX A: FULL PROOFS
% =============================================================================
\section{Full Proofs}
\label{app:proofs}

\subsection{Proof of \cref{prop:error_amplification}}
\label{app:proof_error_amplification}

\begin{proof}
\textbf{$\epsilon$-prediction.}
Starting from the forward process $\zt = tx + (1-t)\epsilon$, we have:
\begin{align}
    \xhat^{(\epsilon)} &= \frac{\zt - (1-t)\epstheta}{t} = \frac{tx + (1-t)\epsilon - (1-t)\epstheta}{t} = x + \frac{1-t}{t}(\epsilon - \epstheta)
\end{align}
Therefore $\norm{\xhat^{(\epsilon)} - x}_2 = \frac{1-t}{t}\norm{\epsilon - \epstheta}_2 = \frac{1-t}{t}\delta_\epsilon$.

\textbf{$v$-prediction.}
For $v$-prediction, using $\xhat^{(v)} = \zt + (1-t)\vtheta$ and $\epsilon = x - v$:
\begin{align}
    \xhat^{(v)} &= tx + (1-t)(x - v) + (1-t)\vtheta = x + (1-t)(\vtheta - v)
\end{align}
Therefore $\norm{\xhat^{(v)} - x}_2 = (1-t)\delta_v$. The error is attenuated by $(1-t) \leq 1$.

\textbf{$x$-prediction.}
For direct $x$-prediction: $\norm{\xhat^{(x)} - x}_2 = \norm{\xtheta - x}_2 = \delta_x$. No amplification.
\end{proof}

\subsection{Proof of \cref{prop:cumulative_error}}
\label{app:proof_cumulative_error}

\textbf{Full statement (general form).} Under guidance energy $\energy$ with $L_g$-Lipschitz gradient and schedule $\alpha_t \leq \alpha$, guided Euler sampling with $N$ uniform steps from $t_0 > 0$ to $1$ yields cumulative perturbation:
\begin{align}
    \text{$\epsilon$-pred:} \quad C_\epsilon &= \alpha L_g \int_{t_0}^{1} \frac{1-t}{t}\,\delta_\epsilon(t)\, dt \label{eq:cum_eps_general}\\
    \text{$v$-pred:} \quad C_v &= \alpha L_g \int_{t_0}^{1} (1-t)\,\delta_v(t)\, dt \label{eq:cum_v_general}\\
    \text{$x$-pred:} \quad C_x &= \alpha L_g \int_{t_0}^{1} \delta_x(t)\, dt \label{eq:cum_x_general}
\end{align}

\textbf{Constant-error corollary.} When prediction errors are approximately uniform:
\begin{align}
    C_\epsilon &= \alpha L_g \delta_\epsilon \big[(t_0{-}1) - \ln t_0\big] \label{eq:cum_eps}\\
    C_v &= \alpha L_g \delta_v (1-t_0)^2/2 \label{eq:cum_v}\\
    C_x &= \alpha L_g \delta_x (1-t_0) \label{eq:cum_x}
\end{align}
The $\epsilon$-prediction integral diverges as $t_0 \to 0$; $v$-prediction converges quadratically; $x$-prediction converges linearly. When $\delta_v > 2\delta_x/(1{-}t_0)$, $x$-prediction achieves lower cumulative error than $v$-prediction despite both converging.

\begin{proof}
By the Lipschitz assumption, per-step guidance error is $\norm{\grad \energy(\xhat) - \grad \energy(x)}_2 \leq L_g \norm{\xhat - x}_2$. By \cref{prop:error_amplification}, $\norm{\xhat^{(\epsilon)} - x}_2 = \frac{1-t}{t}\delta_\epsilon(t)$, $\norm{\xhat^{(v)} - x}_2 = (1{-}t)\delta_v(t)$, and $\norm{\xhat^{(x)} - x}_2 = \delta_x(t)$. Summing per-step contributions and passing to the continuous limit yields \cref{eq:cum_eps_general,eq:cum_v_general,eq:cum_x_general}. Under constant errors, $\int_{t_0}^{1}\frac{1-t}{t}dt = (t_0 - 1) - \ln t_0$ (diverges as $t_0 \to 0$) and $\int_{t_0}^{1}(1-t)dt = (1-t_0)^2/2$.
\end{proof}

The per-step bound holds for any sampler; Heun and DDPM provide additional error correction, making the Euler integral conservative.

\begin{corollary}[Necessary Condition for Bounded Cumulative Error]
\label{cor:necessary_condition}
For $C_\epsilon$ to remain bounded as $t_0 \to 0$, a necessary condition is $\delta_\epsilon(t) \to 0$ as $t \to 0$. Empirical evidence suggests $\delta_\epsilon(t)$ remains bounded away from zero at high noise~\cite{karras2022edm,hang2023minsnr}; under this condition, the cumulative error diverges. No such constraint applies to $v$- or $x$-prediction.
\end{corollary}

\subsection{Proof of \cref{thm:gradient_stability}}
\label{app:proof_gradient_stability}

\textbf{Full statement.} For a Lipschitz energy function $\energy$ with constant $L$:
\begin{align}
    \norm{\grad_{\zt} \energy(\xhat^{(\epsilon)})}_2 &\leq L \cdot \frac{1}{t}\left(1 + (1-t)\norm{\jacobian_{\epstheta}}_2\right) \label{eq:eps_grad_bound}\\
    \norm{\grad_{\zt} \energy(\xhat^{(v)})}_2 &\leq L \cdot \left(1 + (1-t)\norm{\jacobian_{\vtheta}}_2\right) \label{eq:v_grad_bound}\\
    \norm{\grad_{\zt} \energy(\xhat^{(x)})}_2 &\leq L \cdot \norm{\jacobian_{\xtheta}}_2 \label{eq:x_grad_bound}
\end{align}

\begin{proof}
We analyze the guidance gradient $\grad_{\zt} \energy(\xhat)$ for each prediction target.

\textbf{$x$-prediction.}
$\xhat^{(x)} = \xtheta(\zt, t)$, so by the chain rule:
\begin{equation}
    \grad_{\zt} \energy(\xhat^{(x)}) = \jacobian_{\xtheta}^\top \grad_{\xhat} \energy(\xhat^{(x)})
\end{equation}
Since $\energy$ is Lipschitz with constant $L$: $\norm{\grad_{\zt} \energy(\xhat^{(x)})}_2 \leq \norm{\jacobian_{\xtheta}}_2 \cdot L$. This bound is independent of timestep $t$.

\textbf{$\epsilon$-prediction.}
$\xhat^{(\epsilon)} = \frac{1}{t}(\zt - (1-t)\epstheta)$, giving Jacobian $\frac{\partial \xhat^{(\epsilon)}}{\partial \zt} = \frac{1}{t}(\mI - (1-t)\jacobian_{\epstheta})$. By the chain rule and triangle inequality:
\begin{align}
    \norm{\grad_{\zt} \energy(\xhat^{(\epsilon)})}_2 &\leq \frac{1}{t}\norm{\mI - (1-t)\jacobian_{\epstheta}}_2 \cdot L \leq \frac{L}{t}\left(1 + (1-t)\norm{\jacobian_{\epstheta}}_2\right)
\end{align}
As $t \to 0$, this scales as $\bigO(1/t) \to \infty$.

\textbf{$v$-prediction.}
$\xhat^{(v)} = \zt + (1-t)\vtheta$, giving Jacobian $\mI + (1-t)\jacobian_{\vtheta}$. Taking norms:
\begin{equation}
    \norm{\grad_{\zt} \energy(\xhat^{(v)})}_2 \leq \left(1 + (1-t)\norm{\jacobian_{\vtheta}}_2\right) \cdot L
\end{equation}
Finite for all $t \in [0, 1]$ and decreasing as $t \to 1$.
\end{proof}

These bounds hold for arbitrary $\norm{\jacobian}_2$; the asymptotic scaling ($\bigO(1/t)$, $\bigO(1)$) and numerical estimates in \cref{thm:critical_rho} further assume $\norm{\jacobian}_2 = \bigO(1)$.

\subsection{Adversarial Gradient Analysis (Extension of \cref{thm:gradient_stability})}
\label{app:adversarial_analysis}

\paragraph{Adversarial Gradient Condition.}
Following Shen~\etal~\cite{shen2024tfgunderstanding}, a guidance gradient is \emph{adversarial} if it opposes the true improvement direction:
\begin{equation}
    \inner{\grad_{\zt} \energy(\xhat)}{\grad_{x^*} \energy(x^*)} < 0
    \label{eq:adversarial_def}
\end{equation}
where $x^*$ is the optimal clean sample. Shen~\etal\ show that this probability depends on the recovery Jacobian:
\begin{equation}
    P(\text{adversarial}) \propto \norm{\frac{\partial \xhat}{\partial \zt}}_2^2 \cdot \Var[\grad \energy]
    \label{eq:adversarial_prob_full}
\end{equation}

\paragraph{Prediction-Target Comparison.}
From the Recovery Jacobian bounds:
\begin{align}
    \epsilon\text{-pred: } &P(\text{adv}) \propto \frac{1}{t^2}(1 + (1-t)\norm{\jacobian_{\epstheta}}_2)^2 \cdot \Var[\grad \energy] \\
    v\text{-pred: } &P(\text{adv}) \propto (1 + (1-t)\norm{\jacobian_{\vtheta}}_2)^2 \cdot \Var[\grad \energy] \\
    x\text{-pred: } &P(\text{adv}) \propto \norm{\jacobian_{\xtheta}}_2^2 \cdot \Var[\grad \energy]
\end{align}
At $t = 0.01$, $\epsilon$-prediction's adversarial probability exceeds $x$-prediction's by ${\sim}10{,}000\times$ (assuming $\norm{\jacobian}_2 = \bigO(1)$ across targets). $v$-prediction's bound is finite but exceeds $x$-prediction's by factor $(1{+}(1{-}t)\norm{\jacobian_{\vtheta}})^2/\norm{\jacobian_{\xtheta}}^2$, which under the same assumption is $\approx 4$ at $t = 0.01$.

\subsection{Critical Guidance Strength}
\label{app:proof_critical_rho}

\begin{theorem}[Critical Guidance Strength]
\label{thm:critical_rho}
Let $\rho$ denote guidance strength, $L$ the Lipschitz constant of $\energy$, and $t_{\min}$ the minimum timestep. The maximum guidance strength before manifold departure scales as:
\begin{align}
    \rho^*_\epsilon &\propto t_{\min} / L, \quad
    \rho^*_v \propto 1 / (L \cdot (1 + \norm{\jacobian_{\vtheta}})), \quad
    \rho^*_x \propto 1 / (L \cdot \norm{\jacobian_{\xtheta}}) \label{eq:rho_star}
\end{align}
yielding $\rho^*_x / \rho^*_\epsilon \approx 1/(t_{\min} \cdot \norm{\jacobian_{\xtheta}}) \approx 20$ and $\rho^*_x / \rho^*_v = (1 + \norm{\jacobian_{\vtheta}})/\norm{\jacobian_{\xtheta}} \approx 2$ with $t_{\min} = 0.05$ and $\norm{\jacobian_{\xtheta}}, \norm{\jacobian_{\vtheta}} = \bigO(1)$.
\end{theorem}

The empirical ratio (JiT $\rho{=}8$ vs.\ DiT $\rho{=}0.5$, $16\times$; \cref{sec:results}) is broadly consistent with this prediction (confounds discussed in \suppref{app:confound_discussion}).

\begin{proof}
From \cref{thm:gradient_stability}, guidance perturbation scales as $\bigO(\rho L/t)$ for $\epsilon$-prediction, $\bigO(\rho L (1{+}\norm{\jacobian_{\vtheta}}))$ for $v$-prediction, and $\bigO(\rho L \norm{\jacobian_{\xtheta}})$ for $x$-prediction.

\paragraph{Step 1: Manifold-restoring force.}
The manifold-restoring force scales as $\bigO(1/(1-t)^2)$ (\cref{rem:manifold_force_full}), effectively $\bigO(1)$ at $t_{\min} = 0.05$. At the most vulnerable timestep:
\begin{itemize}
    \item For $\epsilon$-prediction, the bottleneck is $t = t_{\min}$: perturbation $\bigO(\rho L/t_{\min})$ vs.\ restoring force $\approx 1$.
    \item For $v$-prediction, the bottleneck is $t = 0$: perturbation $\bigO(\rho L (1{+}\norm{\jacobian_{\vtheta}}))$, finite but larger than $x$-prediction.
    \item For $x$-prediction, no singular timestep exists: perturbation uniformly $\bigO(\rho L \norm{\jacobian_{\xtheta}})$.
\end{itemize}

\paragraph{Step 2: Critical strength.}
Setting perturbation equal to restoring force $K$ (a prediction-target-independent constant that cancels in the ratios below):
\begin{align}
    \epsilon\text{-pred: } &\rho^*_\epsilon \sim K \cdot t_{\min} / L \\
    v\text{-pred: } &\rho^*_v \sim K / (L \cdot (1 + \norm{\jacobian_{\vtheta}})) \\
    x\text{-pred: } &\rho^*_x \sim K / (L \cdot \norm{\jacobian_{\xtheta}})
\end{align}

\paragraph{Step 3: Ratios.}
\begin{equation}
    \frac{\rho^*_x}{\rho^*_\epsilon} = \frac{1}{t_{\min} \cdot \norm{\jacobian_{\xtheta}}} \approx 20, \qquad
    \frac{\rho^*_x}{\rho^*_v} = \frac{1 + \norm{\jacobian_{\vtheta}}}{\norm{\jacobian_{\xtheta}}} \approx 2
\end{equation}
\end{proof}

\subsection{Proof of \cref{prop:score_error}}
\label{app:proof_score_error}

\begin{proof}
From the score decomposition (\cref{eq:score_decomposition} in \suppref{app:manifold_theory}), the true score is $s^\star(\zt, t) = \frac{1}{(1-t)^2}(t \cdot \xhat(\zt, t) - \zt)$ where $\xhat(\zt, t) = \E[X \mid \zt]$. For any approximation $\tilde{x}(\zt, t)$, the induced score estimate is $\tilde{s}(\zt, t) = \frac{1}{(1-t)^2}(t \cdot \tilde{x}(\zt, t) - \zt)$. Taking the difference:
\begin{equation}
    \norm{\tilde{s}(\zt, t) - s^\star(\zt, t)}_2 = \frac{t}{(1-t)^2} \norm{\tilde{x}(\zt, t) - \xhat(\zt, t)}_2
\end{equation}
by direct substitution.
\end{proof}

% =============================================================================
% APPENDIX E: EVALUATION METRIC JUSTIFICATION
% =============================================================================
\section{Evaluation Metric Justification}
\label{app:metric_justification}

As discussed in \cref{subsec:eval_protocol}, standard FID and classifier accuracy are insufficient for detecting manifold departure under gradient-based guidance. Here we formalize the underlying circular evaluation problem and justify C-FID as a complementary metric.

\paragraph{The circular evaluation problem.}
Training-free guidance computes gradients through an off-the-shelf classifier $\phi$, while evaluation measures accuracy using a (possibly different) classifier $\psi$. When $\phi$ and $\psi$ share similar feature-space biases, artifacts introduced by $\phi$'s gradients may go undetected by $\psi$, creating a \textbf{shared vulnerability}. Shen \etal~\cite{shen2024tfgunderstanding} demonstrate that training-free guidance is more susceptible to such \textbf{adversarial gradients} compared to classifier guidance trained on noisy data, exacerbating this risk.

\paragraph{Child FID as a manifold-aware metric.}
C-FID (\cref{subsec:eval_protocol}) is computed between \emph{all} guided samples (9{,}152 images pooled across 143 species) and the full bird species dataset~\cite{piosenka2023birds} (${\sim}$90{,}000 reference images). This pooled computation ensures sufficient sample size for reliable FID estimation, while the domain-specific reference detects manifold departure invisible to P-FID. Unlike Validity, C-FID captures perceptual quality within the target class: adversarial-like samples that fool the classifier but lack realistic appearance will inflate C-FID even when Validity is high. Combined with guidance-strength Pareto sweeps, this reveals the full quality--guidance trade-off that single-point comparisons obscure.

This connection is not merely analogical: Stutz \etal~\cite{stutz2019disentangling} established that adversarial perturbations push inputs off the data manifold, and Dai \etal~\cite{dai2024advdiff} demonstrated that classifier guidance gradients can be directly repurposed to generate adversarial examples.

\paragraph{Why P-FID misleads.}
FID is a Fréchet distance between Gaussian fits to two feature distributions~\cite{heusel2017fid}; the choice of reference distribution is what it measures distance to. P-FID uses the \emph{full ImageNet marginal} as reference, whereas C-FID uses the \emph{target sub-class} distribution. As gradient-based guidance succeeds, the generated distribution narrows from the broad parent prior toward a single target species, so its distance to the broad ImageNet marginal grows regardless of sample quality: P-FID therefore conflates successful narrowing with genuine manifold drift. Referenced to the target sub-class, C-FID instead falls when narrowing stays on-manifold and rises when it does not, separating the two cases.

The Precision and Recall of \suppref{app:prdc} give a complementary view at a different granularity: a local $k$-nearest-neighbour estimate at the sample level rather than a Gaussian summary at the distribution level. The two agree on the ordering of prediction targets, which indicates the ordering is a property of the data rather than of a particular estimator. Validity is blind to both: it accepts any sample the classifier labels correctly, including off-manifold artifacts.

% =============================================================================
% APPENDIX I: EVALUATION PRACTICES SURVEY
% =============================================================================
\section{Evaluation Practices in Training-Free Guidance Literature}
\label{app:eval_survey}

We surveyed 17 training-free guidance method papers published 2022--2025 to assess evaluation methodology quality. \cref{tab:eval_sample_size,tab:eval_metrics,tab:eval_protocol_rigor} compare evaluation practices across four axes; the final row (\textbf{Ours}) shows our protocol for contrast. All papers surveyed: DPS~\cite{chung2023dps}, FreeDoM~\cite{yu2023freedom}, LGD~\cite{song2023lgd}, UGD~\cite{bansal2023universal}, MPGD~\cite{he2024mpgd}, TFG~\cite{ye2024tfg}, CFG++~\cite{chung2024cfgpp}, SAG~\cite{hong2023sag}, PAG~\cite{ahn2024pag}, SEG~\cite{hong2024seg}, NAG~\cite{chen2025nag}, Flow Guidance~\cite{feng2025flowguidance}, OC-Flow~\cite{wang2025ocflow}, DAS~\cite{kim2025das}, FK Steering~\cite{singhal2025fksteering}, DDRM~\cite{kawar2022ddrm}, $\Pi$GDM~\cite{song2023pigdm}.

\subsection{Cross-Paper Comparison}

\noindent\textbf{Finding} (\cref{tab:eval_sample_size})\textbf{.} Only SAG and PAG use the recommended 50{,}000 samples for FID~\cite{heusel2017fid}; most use 1{,}000--2{,}048, where FID variance is high enough to render cross-method comparisons unreliable.

\begin{table}[tbp]
    \centering
    \caption{\textbf{Sample size and statistical rigor across TFG papers.} NR = not reported.}
    \label{tab:eval_sample_size}
    \scriptsize
    \begin{tabular}{@{}llcc@{}}
        \toprule
        Paper & Venue & FID/KID Samples & FID-50K? \\
        \midrule
        DPS & ICLR'23 & 1{,}000 & No \\
        FreeDoM & ICCV'23 & 1{,}000 & No \\
        LGD & ICML'23 & 10K--50K & Partial \\
        UGD & ICLR'24 & NR & N/A \\
        MPGD & ICLR'24 & 1{,}000 (KID) & No \\
        TFG & NeurIPS'24 & 2{,}048 & No \\
        CFG++ & ICLR'25 & NR & Unknown \\
        SAG & ICCV'23 & 50{,}000 & Yes \\
        PAG & ECCV'24 & 50{,}000 & Yes \\
        SEG & NeurIPS'24 & 30{,}000 & No \\
        NAG & 2025 & 5{,}000 & No \\
        Flow Guid. & ICML'25 & 3{,}000 & No \\
        OC-Flow & ICLR'25 & NR & No \\
        DAS & ICML'25 & NR & N/A \\
        FK Steering & 2025 & NR & N/A \\
        DDRM & NeurIPS'22 & 1{,}000 (KID) & No \\
        $\Pi$GDM & ICLR'23 & $\sim$1{,}000 & No \\
        \midrule
        \textbf{Ours} & & \textbf{9{,}152 / 50K} & \textbf{Yes} \\
        \bottomrule
    \end{tabular}
\end{table}

\noindent\textbf{Finding} (\cref{tab:eval_metrics})\textbf{.} Only 2/17 papers (SAG, PAG) report Precision/Recall, the only manifold-aware metrics any surveyed paper uses; the remaining 15/17 report none. No prior paper reports C-FID or any metric designed to disambiguate guidance success from quality degradation. Critically, while 4/17 papers perform systematic guidance-strength sweeps, \emph{none} combine sweeps with manifold-aware metrics: this combination is required to detect the off-manifold failure mode we identify. We are the first to evaluate guided generation via sweep plots that jointly track Validity and C-FID across guidance strengths.

\begin{table}[tbp]
    \centering
    \caption{\textbf{Metric adequacy across TFG papers.} $\checkmark$ = reported, -- = not reported, $*$ = partial/via follow-up. Sweep = guidance-strength sweep with quality metrics. Circular risk: \textbf{L}ow, \textbf{M}oderate, \textbf{H}igh, \textbf{VH} = very high.}
    \label{tab:eval_metrics}
    \scriptsize
    \begin{tabular}{@{}lcccccccc@{}}
        \toprule
        Paper & FID & IS & LPIPS & PSNR & P/R & C-FID & Sweep & Circ. \\
        \midrule
        DPS & $\checkmark$ & -- & $\checkmark$ & $*$ & -- & -- & $*$ & L \\
        FreeDoM & $\checkmark$ & -- & -- & -- & -- & -- & -- & \textbf{VH} \\
        LGD & $\checkmark$ & -- & -- & -- & -- & -- & $*$ & M \\
        UGD & $*$ & -- & -- & -- & -- & -- & $*$ & \textbf{H} \\
        MPGD & -- & -- & $\checkmark$ & -- & -- & -- & -- & \textbf{H} \\
        TFG & $\checkmark$ & -- & $*$ & -- & -- & -- & $\checkmark$ & \textbf{H} \\
        CFG++ & $\checkmark$ & -- & $\checkmark$ & $\checkmark$ & -- & -- & $*$ & L \\
        SAG & $\checkmark$ & $\checkmark$ & -- & -- & $\checkmark$ & -- & $\checkmark$ & L \\
        PAG & $\checkmark$ & $\checkmark$ & $\checkmark$ & -- & $\checkmark$ & -- & $\checkmark$ & M \\
        SEG & $\checkmark$ & -- & $*$ & -- & -- & -- & $\checkmark$ & M \\
        NAG & $\checkmark$ & -- & -- & -- & -- & -- & $*$ & M \\
        Flow Guid. & $\checkmark$ & -- & $\checkmark$ & $\checkmark$ & -- & -- & $*$ & L \\
        OC-Flow & -- & -- & $\checkmark$ & -- & -- & -- & -- & M \\
        DAS & -- & -- & -- & -- & -- & -- & -- & M \\
        FK Steering & -- & -- & -- & -- & -- & -- & -- & H \\
        DDRM & -- & -- & -- & $\checkmark$ & -- & -- & -- & L \\
        $\Pi$GDM & $\checkmark$ & -- & $*$ & $\checkmark$ & -- & -- & -- & L \\
        \midrule
        \textbf{Ours} & $\checkmark$ & $\checkmark$ & $\checkmark$ & $\checkmark$ & $\checkmark$ & $\checkmark$ & $\checkmark$ & \textbf{L} \\
        \bottomrule
    \end{tabular}
\end{table}

\noindent\textbf{Finding} (\cref{tab:eval_protocol_rigor})\textbf{.} Only 3/17 papers (TFG, PAG, SAG) perform systematic guidance sweeps. Several methods multiply NFE without acknowledgment (LGD: $K{=}20$ MC samples; TFG: up to 16$\times$ via recurrence; DAS: $k$ particles), so without NFE-matched comparisons, reported improvements may reflect additional compute.

\begin{table}[tbp]
    \centering
    \caption{\textbf{Evaluation protocol rigor across TFG papers.} Most papers lack systematic guidance sweeps and NFE-matched comparisons, obscuring quality--guidance tradeoffs.}
    \label{tab:eval_protocol_rigor}
    \scriptsize
    \begin{tabular}{@{}lcccc@{}}
        \toprule
        Paper & Sweep? & Pareto? & NFE Match? & Fair? \\
        \midrule
        DPS & Partial & No & NR & No \\
        FreeDoM & Qual. & No & No & No \\
        LGD & Partial & Implicit & No ($K{\times}$) & No \\
        UGD & Partial & No & No & No \\
        MPGD & No & No & Partial & No \\
        TFG & $\checkmark$ & $\checkmark$ & Partial & Partial \\
        CFG++ & Partial & No & Partial & No \\
        SAG & $\checkmark$ & Partial & Yes & Yes \\
        PAG & $\checkmark$ & Partial & Yes & Mostly \\
        SEG & $\checkmark$ & $\checkmark$ & No & No \\
        NAG & Partial & No & Yes & Yes \\
        Flow Guid. & Moderate & No & NR & No \\
        OC-Flow & Minimal & No & NR & No \\
        DAS & Limited & No & No ($k{\times}$) & No \\
        FK Steering & Limited & No & Partial & Partial \\
        DDRM & None & No & No & No \\
        $\Pi$GDM & None & No & NR & No \\
        \midrule
        \textbf{Ours} & $\checkmark$ & $\checkmark$ & \textbf{Yes} & \textbf{Yes} \\
        \bottomrule
    \end{tabular}
\end{table}

\subsection{Notable Case Studies}

\paragraph{Circular evaluation.}
FreeDoM~\cite{yu2023freedom} uses identical networks (CLIP, BiSeNet) for both guidance energy and evaluation distance. MPGD~\cite{he2024mpgd} uses ArcFace for both guidance and evaluation. At least 4/17 papers have high circular evaluation risk (\cref{tab:eval_metrics}).

\paragraph{Manifold claims without manifold metrics.}
MPGD (``Manifold Preserving Guided Diffusion'') and CFG++ (``Manifold-constrained Classifier Free Guidance'') include manifold-related claims in their titles but report zero manifold-aware metrics: no Precision or Recall.

\paragraph{Good practice.}
PAG~\cite{ahn2024pag} and SAG~\cite{hong2023sag} set a higher standard: 50{,}000 FID samples, Precision/Recall, multi-point guidance sweeps, and (for SAG) human evaluation. These practices remain the exception.

% =============================================================================
% APPENDIX F: ARCHITECTURE CONFOUND DISCUSSION
% =============================================================================
\section{Models and the Architecture-Confound Discussion}
\label{app:confound_discussion}

\subsection{JiT Model Variants}
\label{app:jit_variants}

JiT~\cite{li2025jit} (\emph{Back to Basics: Let Denoising Generative Models Denoise}) is a pixel-space flow matching model that directly predicts the clean image $x$ instead of noise $\epsilon$ or velocity $v$. The architecture uses a Vision Transformer backbone with Bottleneck Patch Embedding, RoPE positional encoding, adaLN-Zero conditioning, and SwiGLU feedforward layers. All variants share the same architecture and differ only in depth and width (\cref{tab:jit_variants_app}). JiT is trained for $x$-prediction; in its high-dimensional pixel-space setting, $\epsilon$- and $v$-prediction perform far worse under identical training (a ${>}40\times$ FID gap, analyzed below).

\paragraph{Scale variants.}
We evaluate three JiT variants (B/L/H) to study the interaction between model capacity and guidance quality. JiT-G/16 (2B parameters) is included in \cref{tab:jit_variants_app} for completeness but excluded from guidance experiments: we use JiT-H/16 for parameter-scale comparability with DiT/SiT-XL, and JiT-G's FID improvement over JiT-H is in any case negligible (1.82 vs.\ 1.86) at $2\times$ the computational cost.

\begin{table}[tbp]
    \caption{
        \textbf{JiT model variants.}
        All variants use $x$-prediction in pixel space with identical architecture (patch size 16). FID and IS are CFG-only baselines on ImageNet 256$\times$256.
    }
    \label{tab:jit_variants_app}
    \centering
    \begin{tabular}{@{}lcccc@{}}
        \toprule
        Model & Params & GFLOPs & FID$\downarrow$ & IS$\uparrow$ \\
        \midrule
        JiT-B/16 & 131M & 25 & 3.66 & 275.1 \\
        JiT-L/16 & 459M & 88 & 2.36 & 298.5 \\
        JiT-H/16 & 953M & 182 & 1.86 & 303.4 \\
        JiT-G/16 & 2B & 383 & 1.82 & 292.6 \\
        \bottomrule
    \end{tabular}
\end{table}

\subsection{Why the Comparison Is Not Confounded}

A natural concern is that our comparison of DiT ($\epsilon$-prediction, latent space), SiT ($v$-prediction, latent space), and JiT ($x$-prediction, pixel space) confounds prediction target with operating space and model size. We argue that this ``confound'' is itself evidence for our thesis.

\subsection{Why $\epsilon$-Prediction Requires Architectural Support}

Among publicly available models with official weights, ADM-G~\cite{dhariwal2021diffusion} (2021) remains the \emph{only} pixel-space $\epsilon$-prediction baseline for ImageNet 256$\times$256. Subsequent work universally adopted one of three strategies rather than continuing pure pixel-space $\epsilon$-prediction:
\begin{enumerate}[leftmargin=*,itemsep=2pt,parsep=0pt]
    \item \textbf{Latent compression}: DiT~\cite{peebles2023dit} and Stable Diffusion~\cite{rombach2022ldm} operate in VAE latent space (4,096 dimensions vs.\ 196,608 pixel dimensions).
    \item \textbf{Cascaded generation}: Imagen~\cite{saharia2022imagen} and DALL-E 2 generate at low resolution first.
    \item \textbf{Alternative targets}: PixelFlow~\cite{chen2025pixelflow} uses $v$-prediction; JiT~\cite{li2025jit} uses $x$-prediction.
\end{enumerate}

This pattern is consistent with the dimension scaling argument (\cref{rem:dimension_scaling}): $\epsilon$-prediction requires resolving all $D$ ambient dimensions of the noise, giving base prediction error $\|\delta_\epsilon\|_2 \sim \sqrt{D}$.

\subsection{Controlled Evidence from JiT}
\label{subsec:controlled_evidence}

The main body (\cref{subsec:pretrained_models}) cites the $43\times$ FID gap between $x$- and $\epsilon$-prediction under identical pixel-space training~\cite{li2025jit}. The full three-way ablation (768-dimensional patches, identical architecture) additionally shows $v$-prediction at FID 96.53, intermediate between $x$ (8.62) and $\epsilon$ (372.38), confirming the hierarchy is fundamental, not architectural. This ablation is also why we do not retrain JiT with $\epsilon$- or $v$-prediction for a fully controlled single-architecture comparison at ImageNet scale: such variants are non-functional in pixel space, where the $>$40$\times$ FID gap makes the generated images unsuitable for any guidance evaluation, a fundamental dimension-dependent failure ($D{=}196{,}608$) rather than a tuning issue (consistent with \cref{rem:dimension_scaling}). In lower-dimensional latent spaces all three targets remain competitive (DiT, SiT), confirming the failure is dimension- rather than architecture-dependent.

\subsection{Capacity-Reversed Comparison: JiT-B vs.\ DiT/SiT}

A direct test of the capacity confound: JiT-B/16 ($x$-prediction, 131M parameters) achieves its best C-FID of 31.3 (at $\rho{=}6$), surpassing both DiT-XL/2 ($\epsilon$, 675M, best C-FID 36.7) and SiT-XL/2 ($v$, 675M, best C-FID 34.4). The 5.2$\times$ parameter disadvantage rules out model capacity as the explanation for $x$-prediction's superior guidance quality. The same pattern holds across tasks: JiT-B achieves LPIPS 0.214 on Gaussian deblur ($\rho{=}16$) versus DiT's best LPIPS 0.377 ($\rho{=}0.25$), despite being 5.2$\times$ smaller.

\subsection{Convergent Evidence and Scope}

No single comparison is perfectly controlled, but five independent lines converge on the same hierarchy:
\begin{enumerate}[leftmargin=*,itemsep=2pt,parsep=0pt]
    \item Crossed-lines ablation (fully controlled, identical architecture; \cref{subsec:crossed_lines_results}).
    \item DiT vs.\ SiT (controlled latent pair, $\epsilon < v$).
    \item JiT-B vs.\ DiT (capacity-reversed, 131M $x$ beats 675M $\epsilon$).
    \item PixelFlow C-FID reversal (same pixel space as JiT, $v < x$; \cref{subsec:bird_results}).
    \item Consistent ordering across four tasks (birds, style, deblur, super-resolution; \suppref{app:experimental_protocols}).
\end{enumerate}
The conjunction is difficult to explain by any single confound.

\paragraph{Scope.} Our analysis applies to gradient-based TFG methods, those computing $\grad_{\zt} \energy(\xhat)$, including DPS~\cite{chung2023dps}, LGD~\cite{song2023lgd}, TFG~\cite{ye2024tfg}, FreeDoM~\cite{yu2023freedom}, and Flow Guidance~\cite{feng2025flowguidance}. Attention-based methods (SAG~\cite{hong2023sag}, PAG~\cite{ahn2024pag}, NAG~\cite{chen2025nag}, SEG~\cite{hong2024seg}) do not compute gradients through $\xhat$ and are outside the scope of our error amplification hierarchy.

% =============================================================================
% APPENDIX G: DPS ALGORITHM AND LATENT-SPACE DETAILS
% =============================================================================
\section{DPS Algorithm and Latent-Space Details}
\label{app:algorithm_latent}

\begin{algorithm}[ht]
   \caption{DPS for $x$-Prediction Flow Matching Models}
   \label{alg:dps_flow}
\begin{algorithmic}
   \STATE {\bfseries Input:} Unconditional Clean-Prediction Flow Model $F_\theta$, Guidance Target $y$, Guidance Strength $\rho$, Steps $N_{step}$.
   \STATE $x_0 \sim \mathcal{N}(0, I)$ \COMMENT{Initial Noise}
   \STATE $\Delta t = 1/N_{step}$
   \FOR{$i = 0$ {\bfseries to} $N_{step}-1$}
       \STATE $t \leftarrow i/N_{step}$
       \STATE $\xhat = F_\theta(x_t, t)$ \COMMENT{Clean data estimate}
       \STATE $g_t = \rho \nabla_{x_t} \log p(y \mid \xhat)$ \COMMENT{Guidance gradient}
       \STATE $v_t = (\xhat - x_t)/(1-t)$
       \STATE $z_{next} = x_t + v_t \cdot \Delta t$
       \STATE $x_{t+\Delta t} = z_{next} + \frac{t+\Delta t}{t} g_t$ \COMMENT{$t$ clamped to $t_\epsilon$ for stability}
   \ENDFOR
   \STATE {\bfseries Output} $x_1$
\end{algorithmic}
\end{algorithm}

We use DPS, the minimal form of gradient-based guidance, for reasons discussed in \suppref{app:experimental_protocols}. The $(t{+}\Delta t)/t$ factor on the guidance term $g_t$ is a flow-ODE integration scaling applied identically to all prediction targets (the $t_\epsilon$ clamp bounds it near $t{=}0$), distinct from the target-specific error amplification of \cref{subsec:error_propagation}: that amplification enters only through how the clean estimate $\xhat$ is recovered, the direct network output for $x$-prediction versus the $1/t$ division for $\epsilon$-prediction.

\paragraph{DPS in latent space.}
For latent diffusion models (DiT, SiT), guidance with pixel-space objectives requires VAE decoding. The guidance gradient becomes:
\begin{equation}
    g_t = \rho \nabla_{z_t} \log p(y \mid D(\hat{z}))
    \label{eq:latent_guidance}
\end{equation}
where $D$ is the frozen VAE decoder and $\hat{z}$ is the latent clean estimate.

This introduces three sources of overhead compared to pixel-space guidance:
\begin{enumerate}[leftmargin=*,itemsep=2pt,parsep=0pt]
    \item \textbf{Computational overhead:} Decoder forward pass required for every guidance step
    \item \textbf{Memory overhead:} Decoder gradients must be stored for backpropagation through $D$
    \item \textbf{Potential reconstruction error:} VAE reconstruction artifacts may affect guidance quality
\end{enumerate}

For pixel-space $x$-prediction (JiT), guidance operates directly:
\begin{equation}
    g_t = \rho \nabla_{x_t} \log p(y \mid \xhat)
    \label{eq:pixel_guidance}
\end{equation}
with no decode step required. This gives pixel-space models an efficiency advantage beyond the prediction target effects analyzed in \cref{subsec:error_propagation}.

% =============================================================================
% APPENDIX C: DETAILED EXPERIMENTAL PROTOCOLS
% =============================================================================
\section{Experimental Protocols and Full Results}
\label{app:experimental_protocols}

\subsection{Factorial Design Rationale}

Our experiments use a \textbf{2$\times$3 factorial design}: three prediction targets ($\epsilon$, $v$, $x$) crossed with two operating spaces (pixel, latent). Model selection rationale and confound analysis are in \cref{subsec:pretrained_models} and \suppref{app:confound_discussion}.

\paragraph{Time Convention Handling.}
Models use different time conventions:
\begin{itemize}[leftmargin=*,itemsep=2pt,parsep=0pt]
    \item \textbf{DDPM} (ADM-G, DiT): $t=0$ is clean data, $t=T$ (999) is noise
    \item \textbf{Flow matching} (SiT, PixelFlow, JiT): $t=0$ is noise, $t=1$ is clean data
\end{itemize}
Our guidance implementation normalizes all models to flow matching convention internally.

\subsection{Bird Species Dataset}
\label{app:bird_dataset}

Our fine-grained bird benchmark (\cref{subsec:bird_benchmark}) uses the 525 Bird Species dataset~\cite{piosenka2023birds}, a CC0 (Public Domain) image classification dataset published on Kaggle and mirrored on HuggingFace.

\paragraph{Curation.}
The dataset contains approximately 90{,}000 images across 525 bird species (at least 130 training images per species, plus 5 test and 5 validation images each). Images were collected from internet searches by species name, deduplicated using automated detection, and cropped so that the bird occupies at least 50\% of pixels. All images are resized to $224 \times 224$ RGB JPEGs. Each species includes a scientific name.

\paragraph{Prior usage.}
Ye~\etal~\cite{ye2024tfg} used this dataset and the same EfficientNetB2 classifier for the first fine-grained label guidance study (TFG, NeurIPS 2024); we extend their single-model evaluation to a systematic cross-model, multi-strength Pareto analysis.

\paragraph{Our usage.}
Of the 525 species, 143 map to 30 ImageNet parent classes (2--20 species per parent, mean 4.8). To our knowledge, this hierarchical species-to-ImageNet mapping, which enables two-level conditioning (CFG for parent class, DPS for species) within a single generation pipeline, has not been established in prior work. We use separate classifiers for guidance (EfficientNetB2\footnote{\url{https://huggingface.co/dennisjooo/Birds-Classifier-EfficientNetB2}}) and evaluation\footnote{\url{https://huggingface.co/chriamue/bird-species-classifier}} to avoid circular evaluation (\suppref{app:metric_justification}). The species-to-ImageNet mapping and evaluation code will be publicly released.

\subsection{DPS Guidance Comparison}

\paragraph{Guidance Setup.}
We apply DPS~\cite{chung2023dps} to each model, sweeping guidance strength $\rho$ while keeping the method minimal: no mean guidance ($\mu{=}0$), no recurrence ($N_{\text{recur}}{=}1$), no Monte Carlo smoothing ($\sigma{=}0$). This isolates the effect of prediction target on gradient quality without confounding by auxiliary hyperparameters.

\paragraph{Task.}
Fine-grained bird classification on our hierarchical benchmark (143 species, 30 parent classes; see \cref{subsec:bird_benchmark}). CFG steers toward the parent class; DPS guides toward a specific species via an external classifier. We use separate classifiers for guidance and evaluation.

\subsection{Full Results: Fine-Grained Bird Classification}
\label{app:finegrained_full_results}

\cref{tab:finegrained_full} reports all numerical results for the guidance-strength sweep in \cref{subsec:bird_results}. Each row corresponds to a single ($\text{model}, \rho$) configuration; each data point represents 9{,}152 generated images (143 species $\times$ 64 samples). Models are grouped by prediction target: $\epsilon$-prediction (DiT), $v$-prediction (SiT, PixelFlow), and $x$-prediction (JiT variants).

\begin{table}[tbp]
    \centering
    \caption{\textbf{Fine-grained bird classification: full guidance-strength sweep.} P-FID: FID against full ImageNet reference. C-FID: FID against bird species dataset. Validity: top-1 accuracy on evaluation classifier. $\rho{=}0$ denotes CFG-only baseline.}
    \label{tab:finegrained_full}
    \footnotesize
    \begin{tabular}{@{}llrrrr@{}}
        \toprule
        Model & & $\rho$ & P-FID$\downarrow$ & C-FID$\downarrow$ & Validity(\%)$\uparrow$ \\
        \midrule
        DiT-XL/2 & & 0 & 5.52 & 42.81 & 14.13 \\
        ($\epsilon$, latent) & & 0.05 & 6.09 & 39.35 & 23.90 \\
        & & 0.10 & 6.70 & 38.11 & 26.69 \\
        & & 0.25 & 9.47 & 37.61 & 28.81 \\
        & & 0.50 & 14.22 & 36.66 & 29.63 \\
        \midrule
        SiT-XL/2 & & 0 & 5.07 & 41.58 & 13.68 \\
        ($v$, latent) & & 0.05 & 5.77 & 38.46 & 20.74 \\
        & & 0.10 & 6.50 & 38.15 & 22.34 \\
        & & 0.25 & 7.12 & 37.41 & 23.73 \\
        & & 0.50 & 7.18 & 35.71 & 25.36 \\
        & & 1.00 & 8.21 & 34.66 & 26.64 \\
        & & 1.50 & 9.92 & 34.38 & 27.35 \\
        & & 2.00 & 12.48 & 34.94 & 27.61 \\
        \midrule
        PixelFlow & & 0 & 6.29 & 44.07 & 13.55 \\
        ($v$, pixel) & & 0.50 & 6.49 & 38.14 & 20.45 \\
        & & 2.00 & 11.69 & 36.22 & 20.13 \\
        & & 3.00 & 16.49 & 38.52 & 19.13 \\
        & & 5.00 & 28.24 & 47.71 & 18.07 \\
        \midrule
        JiT-B/16 & & 0 & 8.84 & 46.43 & 14.94 \\
        ($x$, pixel) & & 0.50 & 7.20 & 40.16 & 22.91 \\
        & & 1.00 & 6.97 & 38.83 & 23.87 \\
        & & 2.00 & 7.27 & 36.42 & 25.54 \\
        & & 6.00 & 12.24 & 31.30 & 26.77 \\
        & & 10.00 & 21.79 & 33.19 & 26.58 \\
        \midrule
        JiT-L/16 & & 0 & 7.03 & 44.01 & 14.29 \\
        ($x$, pixel) & & 0.50 & 6.17 & 38.44 & 23.01 \\
        & & 1.00 & 6.37 & 37.18 & 24.32 \\
        & & 2.00 & 6.90 & 35.84 & 25.21 \\
        & & 6.00 & 10.78 & 31.43 & 27.56 \\
        & & 10.00 & 18.06 & 32.21 & 26.19 \\
        \midrule
        JiT-H/16 & & 0 & 5.48 & 40.98 & 14.01 \\
        ($x$, pixel) & & 0.50 & 5.51 & 36.51 & 23.56 \\
        & & 1.00 & 5.76 & 35.19 & 24.62 \\
        & & 1.50 & 6.22 & 34.45 & 25.09 \\
        & & 2.00 & 6.38 & 33.89 & 25.82 \\
        & & 3.00 & 6.91 & 32.85 & 26.60 \\
        & & 4.00 & 7.76 & 31.63 & 26.57 \\
        & & 6.00 & 9.68 & 31.03 & 26.85 \\
        & & 8.00 & 12.21 & 30.63 & 26.02 \\
        & & 10.00 & 14.78 & 31.40 & 26.57 \\
        \bottomrule
    \end{tabular}
\end{table}

\subsection{Guidance Methods Beyond DPS: LGD and FreeDoM}
\label{app:tfg_family}

\cref{thm:gradient_stability} applies to the family of gradient-based TFG methods that differentiate through the clean-data estimate $\xhat$ (\suppref{app:confound_discussion}). To check that the empirical hierarchy is not specific to DPS, we repeat the bird benchmark with two further members of this family, LGD~\cite{song2023lgd} and FreeDoM~\cite{yu2023freedom}, sweeping guidance strength at the matched NFE${\approx}$100 used for the main benchmark. \cref{fig:tfg_family} shows that under both methods JiT-H ($x$-prediction) retains the lowest C-FID frontier, with $v$- and $\epsilon$-prediction above it, the ordering observed under DPS.

\begin{figure}[t]
    \centering
    \begin{subfigure}[b]{0.49\textwidth}
        \includegraphics[width=\textwidth]{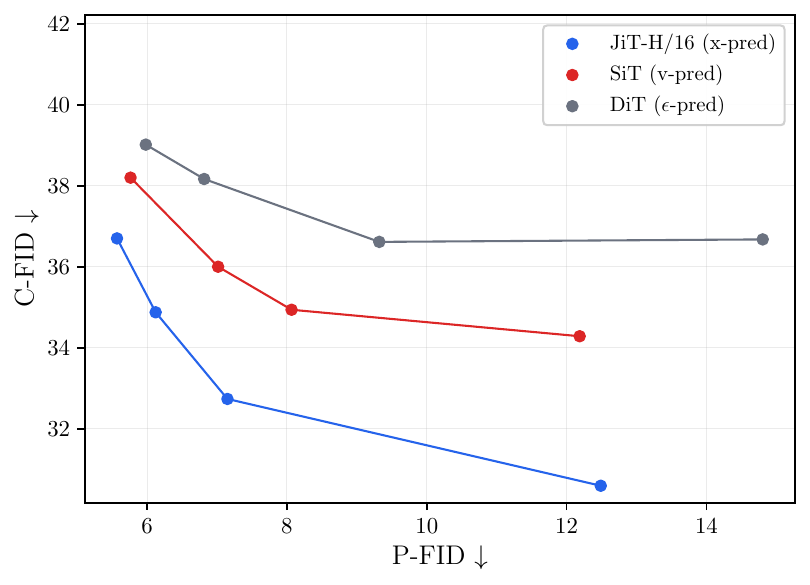}
        \caption{LGD}
        \label{fig:tfg_family_lgd}
    \end{subfigure}
    \hfill
    \begin{subfigure}[b]{0.49\textwidth}
        \includegraphics[width=\textwidth]{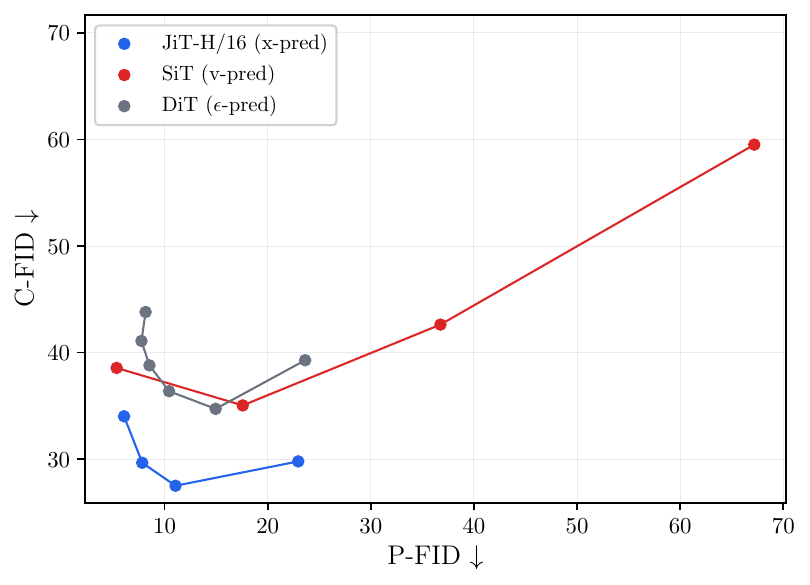}
        \caption{FreeDoM}
        \label{fig:tfg_family_freedom}
    \end{subfigure}
    \caption{
        \textbf{Prediction-target hierarchy under LGD and FreeDoM (fine-grained bird).}
        P-FID vs.\ C-FID as guidance strength $\rho$ increases, for \textbf{(a)}~LGD and \textbf{(b)}~FreeDoM; lower is better on both axes. JiT-H ($x$) attains the lowest C-FID frontier under both, matching the DPS result (\cref{subsec:bird_results}).
    }
    \label{fig:tfg_family}
\end{figure}

\subsection{Second Fine-Grained Domain: Butterfly}
\label{app:butterfly}

To test whether the prediction-target hierarchy generalizes beyond birds, we build a parallel fine-grained benchmark on butterfly species. From a public 100-species butterfly image dataset~\cite{piosenka2023butterfly}, we take a 34-species subset nested under 6 ImageNet butterfly parents (256 images per species), mirroring the parent--child structure of the bird benchmark: CFG steers toward the parent class and DPS toward the species. \cref{fig:butterfly} shows the resulting P-FID vs.\ C-FID sweep. As on birds, JiT-H ($x$-prediction) attains the lowest C-FID frontier, with $v$- and $\epsilon$-prediction above it.

\begin{figure}[t]
    \centering
    \includegraphics[width=0.62\textwidth]{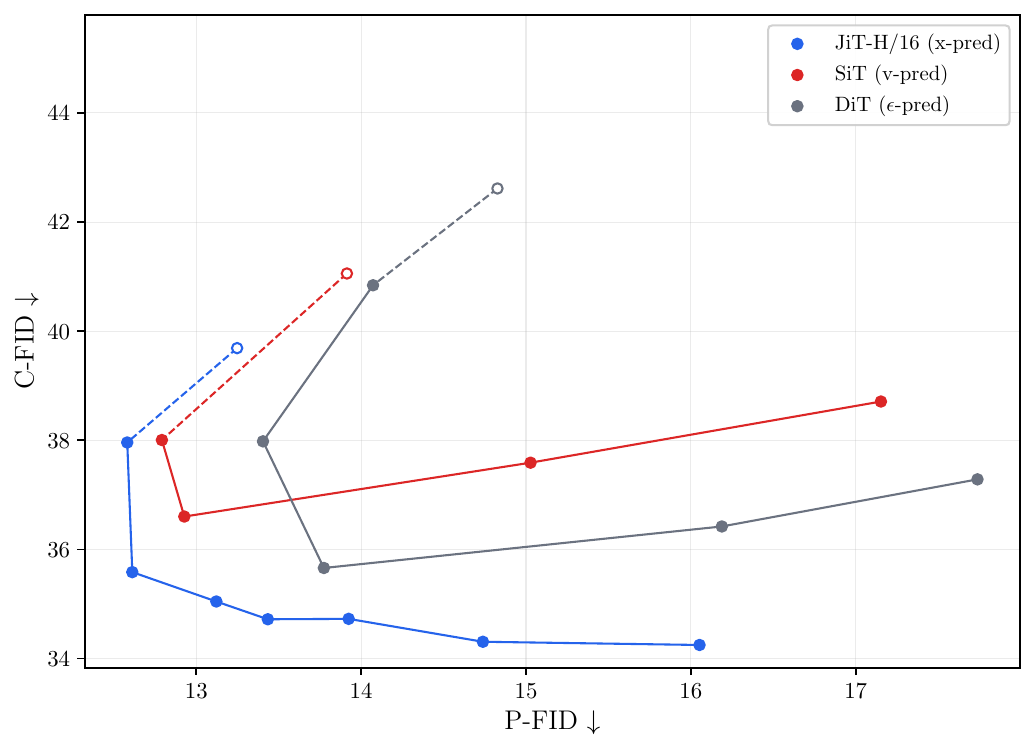}
    \caption{
        \textbf{Prediction-target hierarchy on a second fine-grained domain (butterfly).}
        P-FID vs.\ C-FID over the DPS guidance-strength sweep (34 butterfly species under 6 ImageNet parents); open markers are CFG-only baselines, dashed segments connect to the first guided setting. Lower is better on both axes. JiT-H ($x$) attains the lowest C-FID frontier, matching the bird result.
    }
    \label{fig:butterfly}
\end{figure}

\subsection{Full Results: Style Transfer}
\label{app:style_full_results}

\cref{tab:style_full} reports all numerical results for the style transfer guidance-strength sweep in \cref{subsec:style_results}. Each data point represents 400 generated images (4 WikiArt styles $\times$ 100 ImageNet classes). Models are grouped by prediction target.

\begin{table}[tbp]
    \centering
    \caption{\textbf{Style transfer: full guidance-strength sweep.} Gram Distance: L2 distance between CLIP ViT-B/32 Gram matrices (lower = stronger style match). Content Acc.: DeiT-Small top-1 accuracy on original ImageNet class (content preservation). $\rho{=}0$ denotes CFG-only baseline.}
    \label{tab:style_full}
    \footnotesize
    \begin{tabular}{@{}llrrr@{}}
        \toprule
        Model & & $\rho$ & Gram Dist.$\downarrow$ & Content Acc.(\%)$\uparrow$ \\
        \midrule
        DiT-XL/2 & & 0 & 5.845 & 89.00 \\
        ($\epsilon$, latent) & & 0.10 & 5.769 & 87.00 \\
        & & 0.50 & 5.617 & 85.75 \\
        & & 1.00 & 5.512 & 78.50 \\
        & & 10.00 & 5.296 & 1.50 \\
        \midrule
        SiT-XL/2 & & 0 & 5.901 & 91.00 \\
        ($v$, latent) & & 0.10 & 5.781 & 84.50 \\
        & & 0.50 & 5.590 & 85.25 \\
        & & 1.00 & 5.486 & 85.00 \\
        & & 4.00 & 5.316 & 71.25 \\
        & & 10.00 & 5.250 & 38.25 \\
        \midrule
        PixelFlow & & 0 & 5.904 & 88.00 \\
        ($v$, pixel) & & 0.10 & 5.845 & 85.25 \\
        & & 0.50 & 5.678 & 82.75 \\
        & & 1.00 & 5.555 & 77.00 \\
        & & 2.00 & 5.344 & 65.50 \\
        & & 4.00 & 5.043 & 37.00 \\
        \midrule
        JiT-B/16 & & 0 & 5.914 & 88.00 \\
        ($x$, pixel) & & 0.25 & 5.831 & 92.00 \\
        & & 1.00 & 5.732 & 89.75 \\
        & & 4.00 & 5.576 & 84.75 \\
        & & 10.00 & 5.416 & 76.50 \\
        & & 25.00 & 5.257 & 54.00 \\
        & & 50.00 & 5.181 & 39.00 \\
        \midrule
        JiT-L/16 & & 0 & 5.906 & 94.00 \\
        ($x$, pixel) & & 0.25 & 5.829 & 90.00 \\
        & & 1.00 & 5.726 & 87.75 \\
        & & 4.00 & 5.580 & 83.50 \\
        & & 10.00 & 5.437 & 79.50 \\
        & & 25.00 & 5.287 & 62.00 \\
        & & 50.00 & 5.217 & 47.00 \\
        \midrule
        JiT-H/16 & & 0 & 5.920 & 87.00 \\
        ($x$, pixel) & & 0.25 & 5.810 & 90.50 \\
        & & 1.00 & 5.725 & 92.25 \\
        & & 4.00 & 5.598 & 87.00 \\
        & & 10.00 & 5.484 & 80.00 \\
        & & 25.00 & 5.345 & 63.50 \\
        & & 50.00 & 5.261 & 49.50 \\
        \bottomrule
    \end{tabular}
\end{table}

\subsection{Style Transfer: Qualitative Comparison}
\label{app:style_qualitative}

\cref{fig:style_vis} compares DiT ($\rho{=}1$) and JiT-H ($\rho{=}10$). The failure modes mirror the fine-grained classification pattern (\cref{fig:qualitative}): DiT shows the mode-collapse signature (a narrow set of repeated templates) while JiT-H preserves compositional diversity, consistent with the Precision/Recall analysis (\suppref{app:prdc}).

\begin{figure}[tbp]
    \centering
    \begin{subfigure}[b]{\textwidth}
        \includegraphics[width=\textwidth]{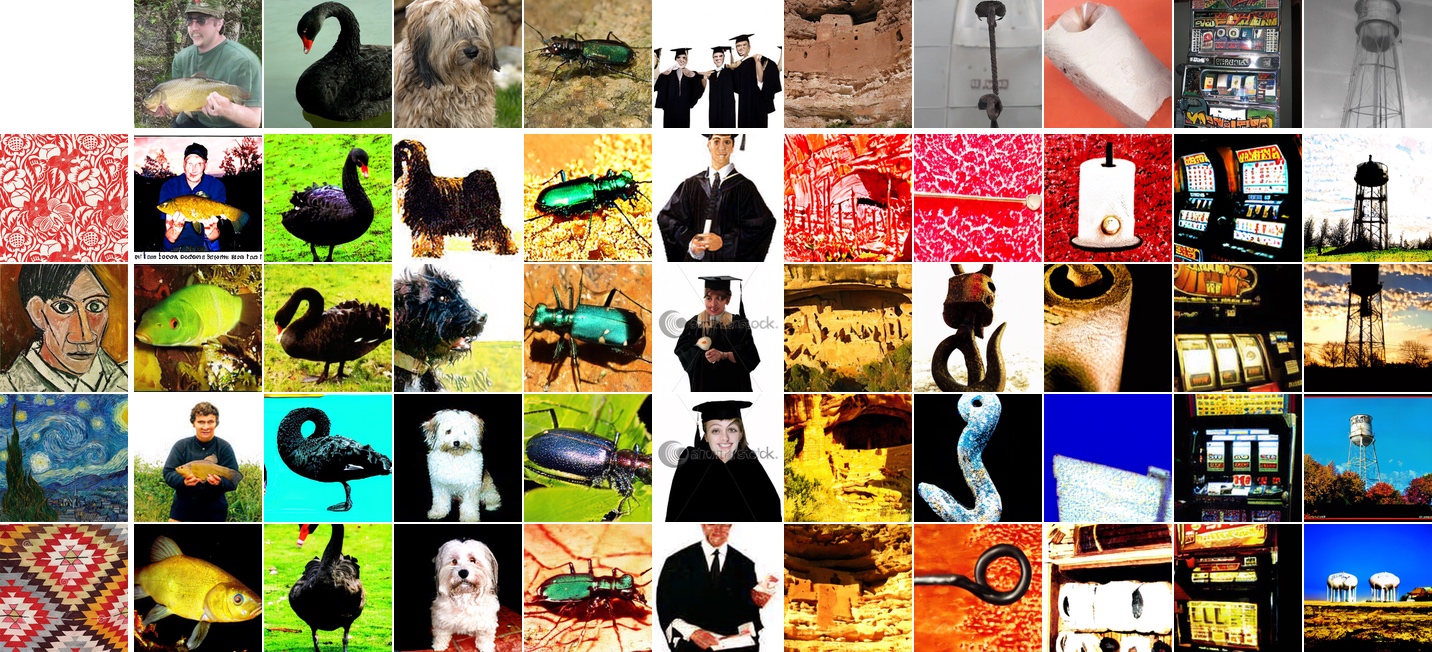}
        \caption{DiT ($\epsilon$, $\rho{=}1$): Gram Dist.\ 5.51, Content Acc.\ 78.5\%}
        \label{fig:style_vis_dit}
    \end{subfigure}

    \vspace{2pt}

    \begin{subfigure}[b]{\textwidth}
        \includegraphics[width=\textwidth]{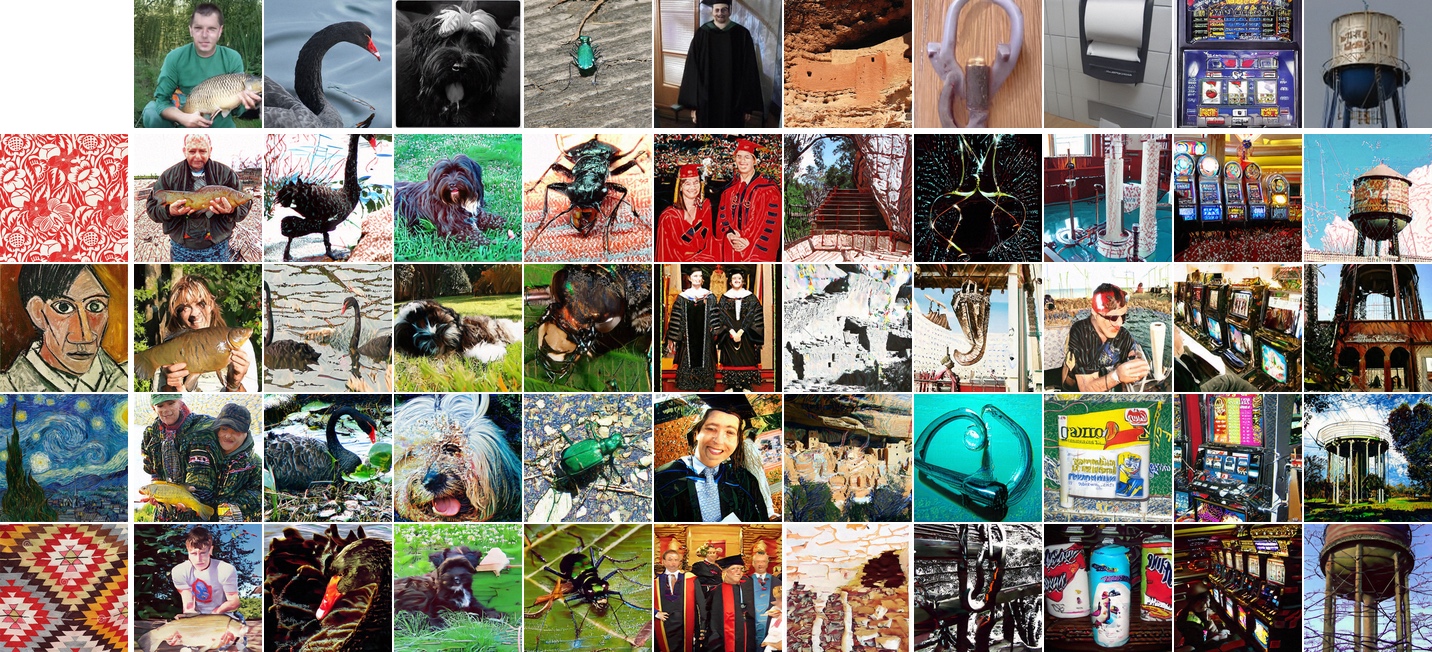}
        \caption{JiT-H ($x$, $\rho{=}10$): Gram Dist.\ 5.48, Content Acc.\ 80.0\%}
        \label{fig:style_vis_jit}
    \end{subfigure}

    \caption{
        \textbf{Style transfer guided generation (10 ImageNet classes $\times$ 4 WikiArt styles).}
        Top row of each panel: CFG-only baseline (no style guidance). Rows 2--5: guided toward each of the four target styles. DiT shows mode concentration and uniform backgrounds; JiT-H maintains more diverse compositions.
    }
    \label{fig:style_vis}
\end{figure}

\subsection{Inverse Problems: Gaussian Deblur and Super-Resolution}
\label{app:inverse_problems}

We additionally evaluate DPS guidance on two inverse problems: Gaussian deblurring (kernel size 61, $\sigma{=}3.0$) and 4$\times$ bicubic super-resolution, following the protocol of \cite{chung2023dps}. For each task, 1{,}000 ImageNet validation images are generated and metrics (LPIPS~\cite{zhang2018lpips}, PSNR, SSIM) are computed on a 100-image subset; DPS guides the denoising process to reconstruct images consistent with the degraded observation.

\paragraph{Baseline performance.}
The identity baseline (returning the degraded input) gives PSNR 21.78 / LPIPS 0.5364 for deblur, and the bicubic baseline gives PSNR 23.66 / LPIPS 0.3929 for super-resolution. No model beats these PSNR baselines, so we compare prediction targets on perceptual quality (LPIPS).

\paragraph{Prediction target comparison.}
Despite the PSNR limitation, JiT ($x$-prediction) achieves the best LPIPS across both tasks: 0.2140 for deblur ($\rho{=}16$) and 0.1886 for super-resolution ($\rho{=}24$). DiT ($\epsilon$-prediction) degrades sharply at strong guidance: LPIPS \emph{increases} beyond $\rho{=}0.25$ for deblur and $\rho{=}1$ for super-resolution, consistent with the error amplification hierarchy. SiT ($v$-prediction) and PixelFlow achieve competitive LPIPS at moderate $\rho$ but plateau or degrade at stronger guidance, while JiT continues improving across a wider $\rho$ range. Full sweep results are given in \cref{tab:deblur_full} (deblur) and \cref{tab:sr_full} (super-resolution).

\begin{table}[tbp]
    \centering
    \caption{\textbf{Gaussian deblur (DPS): full guidance-strength sweep.} Identity baseline: LPIPS 0.5364, PSNR 21.78, SSIM 0.4742. $\rho$ values are model-specific.}
    \label{tab:deblur_full}
    \footnotesize
    \begin{tabular}{@{}llrrrr@{}}
        \toprule
        Model & & $\rho$ & LPIPS$\downarrow$ & PSNR$\uparrow$ & SSIM$\uparrow$ \\
        \midrule
        DiT-XL/2 & & 0.05 & 0.4398 & 15.60 & 0.3086 \\
        ($\epsilon$, latent) & & 0.10 & 0.4062 & 16.92 & 0.3351 \\
        & & 0.25 & 0.3771 & 18.73 & 0.4009 \\
        & & 0.50 & 0.4757 & 17.87 & 0.3686 \\
        & & 1.00 & 0.6639 & 11.87 & 0.2122 \\
        & & 2.00 & 0.6939 & 11.43 & 0.1329 \\
        & & 4.00 & 0.6603 & 11.99 & 0.1359 \\
        \midrule
        SiT-XL/2 & & 0.25 & 0.3995 & 18.28 & 0.3694 \\
        ($v$, latent) & & 0.50 & 0.3403 & 19.66 & 0.4381 \\
        & & 1.00 & 0.3490 & 20.74 & 0.4951 \\
        & & 2.00 & 0.4335 & 20.58 & 0.5310 \\
        & & 4.00 & 0.4797 & 19.31 & 0.5203 \\
        \midrule
        PixelFlow & & 0.25 & 0.4353 & 13.91 & 0.3119 \\
        ($v$, pixel) & & 0.50 & 0.3699 & 15.53 & 0.3604 \\
        & & 1.00 & 0.2977 & 17.40 & 0.4083 \\
        & & 2.00 & 0.2526 & 18.51 & 0.4319 \\
        & & 4.00 & 0.2561 & 17.57 & 0.4052 \\
        \midrule
        JiT-B/16 & & 1.00 & 0.5682 & 9.72 & 0.2006 \\
        ($x$, pixel) & & 2.00 & 0.4155 & 16.69 & 0.3531 \\
        & & 4.00 & 0.3038 & 20.38 & 0.4855 \\
        & & 8.00 & 0.2350 & 21.56 & 0.5633 \\
        & & 16.00 & 0.2140 & 21.28 & 0.5729 \\
        \midrule
        JiT-L/16 & & 1.00 & 0.5724 & 9.94 & 0.2083 \\
        ($x$, pixel) & & 2.00 & 0.4058 & 17.19 & 0.3622 \\
        & & 4.00 & 0.2934 & 20.26 & 0.4913 \\
        & & 8.00 & 0.2243 & 21.51 & 0.5567 \\
        & & 16.00 & 0.2143 & 21.12 & 0.5655 \\
        \midrule
        JiT-H/16 & & 1.00 & 0.5736 & 9.88 & 0.2053 \\
        ($x$, pixel) & & 2.00 & 0.4130 & 16.43 & 0.3376 \\
        & & 4.00 & 0.2960 & 20.06 & 0.4702 \\
        & & 8.00 & 0.2296 & 20.91 & 0.5256 \\
        & & 16.00 & 0.2278 & 20.10 & 0.5251 \\
        \bottomrule
    \end{tabular}
\end{table}

\begin{table}[tbp]
    \centering
    \caption{\textbf{4$\times$ super-resolution (DPS): full guidance-strength sweep.} Bicubic baseline: LPIPS 0.3929, PSNR 23.66, SSIM 0.6465. $\rho$ values are model-specific.}
    \label{tab:sr_full}
    \footnotesize
    \begin{tabular}{@{}llrrrr@{}}
        \toprule
        Model & & $\rho$ & LPIPS$\downarrow$ & PSNR$\uparrow$ & SSIM$\uparrow$ \\
        \midrule
        DiT-XL/2 & & 0.01 & 0.5725 & 10.33 & 0.1993 \\
        ($\epsilon$, latent) & & 0.05 & 0.5145 & 12.81 & 0.2386 \\
        & & 1.00 & 0.3683 & 19.11 & 0.4216 \\
        & & 2.00 & 0.4172 & 18.54 & 0.4385 \\
        & & 4.00 & 0.6259 & 12.79 & 0.2716 \\
        & & 8.00 & 0.6479 & 13.08 & 0.2218 \\
        & & 16.00 & 0.5986 & 14.47 & 0.2434 \\
        \midrule
        SiT-XL/2 & & 1.00 & 0.3994 & 18.99 & 0.4052 \\
        ($v$, latent) & & 2.00 & 0.3342 & 20.53 & 0.4878 \\
        & & 4.00 & 0.2985 & 21.72 & 0.5616 \\
        & & 8.00 & 0.3543 & 21.43 & 0.5915 \\
        & & 16.00 & 0.4168 & 19.87 & 0.5807 \\
        \midrule
        PixelFlow & & 1.00 & 0.4684 & 13.87 & 0.3324 \\
        ($v$, pixel) & & 2.00 & 0.3992 & 15.53 & 0.3968 \\
        & & 4.00 & 0.2907 & 18.13 & 0.4896 \\
        & & 8.00 & 0.2092 & 20.14 & 0.5368 \\
        & & 16.00 & 0.2037 & 19.76 & 0.4992 \\
        \midrule
        JiT-B/16 & & 4.00 & 0.5683 & 9.65 & 0.2235 \\
        ($x$, pixel) & & 8.00 & 0.4052 & 17.28 & 0.4197 \\
        & & 16.00 & 0.2327 & 22.78 & 0.6183 \\
        & & 24.00 & 0.1948 & 23.56 & 0.6540 \\
        & & 32.00 & 0.2003 & 23.56 & 0.6477 \\
        \midrule
        JiT-L/16 & & 4.00 & 0.5705 & 9.99 & 0.2210 \\
        ($x$, pixel) & & 8.00 & 0.3947 & 17.80 & 0.4307 \\
        & & 16.00 & 0.2288 & 22.86 & 0.6203 \\
        & & 24.00 & 0.1895 & 23.49 & 0.6525 \\
        & & 32.00 & 0.2003 & 23.44 & 0.6427 \\
        \midrule
        JiT-H/16 & & 4.00 & 0.5718 & 9.93 & 0.2171 \\
        ($x$, pixel) & & 8.00 & 0.3978 & 17.08 & 0.4100 \\
        & & 16.00 & 0.2314 & 22.40 & 0.6057 \\
        & & 24.00 & 0.1886 & 23.32 & 0.6399 \\
        & & 32.00 & 0.1990 & 22.98 & 0.6216 \\
        \bottomrule
    \end{tabular}
\end{table}

\subsection{Toy Experiments}
\label{app:crossed_lines_details}

\paragraph{Crossed-lines setup.}
We train identical residual MLP models (256 hidden units, 5 ResBlocks with LayerNorm and sinusoidal time conditioning) on a 2D crossed-lines dataset. The ground truth distribution $p(x)$ consists of two 1D manifolds in $\R^2$: lines $b{=}a$ (class~0) and $b{=}{-}a$ (class~1). Points are sampled uniformly along each line ($t \sim \mathrm{Uniform}[-2, 2]$) with additive Gaussian noise $\sigma{=}0.1$ perpendicular to the line, totaling 12{,}000 points (6{,}000 per class). Each 2D point $x_{2\mathrm{D}}$ is embedded as $x_D = x_{2\mathrm{D}} \, P$, where $P \in \R^{2 \times D}$ is a fixed column-orthogonal matrix obtained by QR decomposition of a random Gaussian matrix (seed 42). For each $D \in \{2, 8, 32, 128, 512\}$, we train three flow matching models ($\epsilon$, $v$, $x$ prediction targets) for 500 epochs with learning rate $10^{-3}$ and batch size 256 (seed 42). A separate 3-layer MLP classifier ($128 \times 128 \times 128$, 100 epochs) provides the DPS gradient signal.

\paragraph{Task.}
DPS guidance toward Class~1 (target class), starting from noise. Euler ODE sampling with 100 steps, $t \in [0, 1]$. The guidance formula adds $s \cdot \nabla_{z_t} \log p(y{=}1 \mid \xhat(z_t))$ to the velocity prediction at each step.

\paragraph{Metrics.}
\textbf{On-manifold rate}: each generated sample $\xhat_D \in \R^D$ is first back-projected to 2D via $\xhat_{2\mathrm{D}} = \xhat_D P^\top$. The perpendicular distance to the target line is $\abs{a - b}/\sqrt{2}$ for class~0 ($b{=}a$) or $\abs{a + b}/\sqrt{2}$ for class~1 ($b{=}{-}a$), where $(a, b)$ are the back-projected coordinates. The threshold $\delta$ is calibrated as the 95th percentile of ground truth perpendicular distances. On-manifold rate is the fraction of 10{,}000 generated samples with distance ${<}\,\delta$.
Additional metrics: \textbf{target MMD} (Gaussian kernel MMD between generated and target class samples, median heuristic bandwidth), \textbf{KL divergence} (dual KDE estimate), \textbf{class accuracy} (classifier prediction rate for target class).

\paragraph{Full crossed-lines results ($s{=}10$).}
\cref{tab:crossed_lines_full} reports the complete metrics (on-manifold rate, target MMD, class accuracy) across all ambient dimensions, and \cref{fig:crossed_lines_grid} visualizes the generated distributions projected back to 2D.

\begin{table}[tbp]
    \centering
    \caption{\textbf{Crossed-lines full metrics ($s{=}10$, 100 steps, 10{,}000 samples).}}
    \label{tab:crossed_lines_full}
    \footnotesize
    \begin{tabular}{@{}llccccc@{}}
        \toprule
        Metric & Target & $D{=}2$ & $D{=}8$ & $D{=}32$ & $D{=}128$ & $D{=}512$ \\
        \midrule
        \multirow{3}{*}{On-manifold (\%)}
        & $\epsilon$ & 65.8 & 76.7 & 58.5 & 9.1 & 0.5 \\
        & $v$ & 96.2 & 97.5 & 91.9 & 43.7 & 21.5 \\
        & $x$ & \textbf{100} & \textbf{99.9} & \textbf{100} & \textbf{72.4} & \textbf{93.3} \\
        \midrule
        \multirow{3}{*}{Target MMD ($\downarrow$)}
        & $\epsilon$ & .068 & .092 & .062 & .189 & .327 \\
        & $v$ & .051 & .012 & .018 & .009 & .026 \\
        & $x$ & \textbf{.035} & \textbf{.054} & \textbf{.016} & \textbf{.011} & \textbf{.005} \\
        \midrule
        \multirow{3}{*}{Class Acc.\ (\%)}
        & $\epsilon$ & 82.8 & 94.6 & 81.1 & 49.1 & 56.6 \\
        & $v$ & \textbf{100} & \textbf{100} & \textbf{100} & 99.5 & 87.3 \\
        & $x$ & \textbf{100} & \textbf{100} & \textbf{100} & \textbf{100} & \textbf{100} \\
        \bottomrule
    \end{tabular}
\end{table}

\begin{figure}[tbp]
    \centering
    \includegraphics[width=\textwidth]{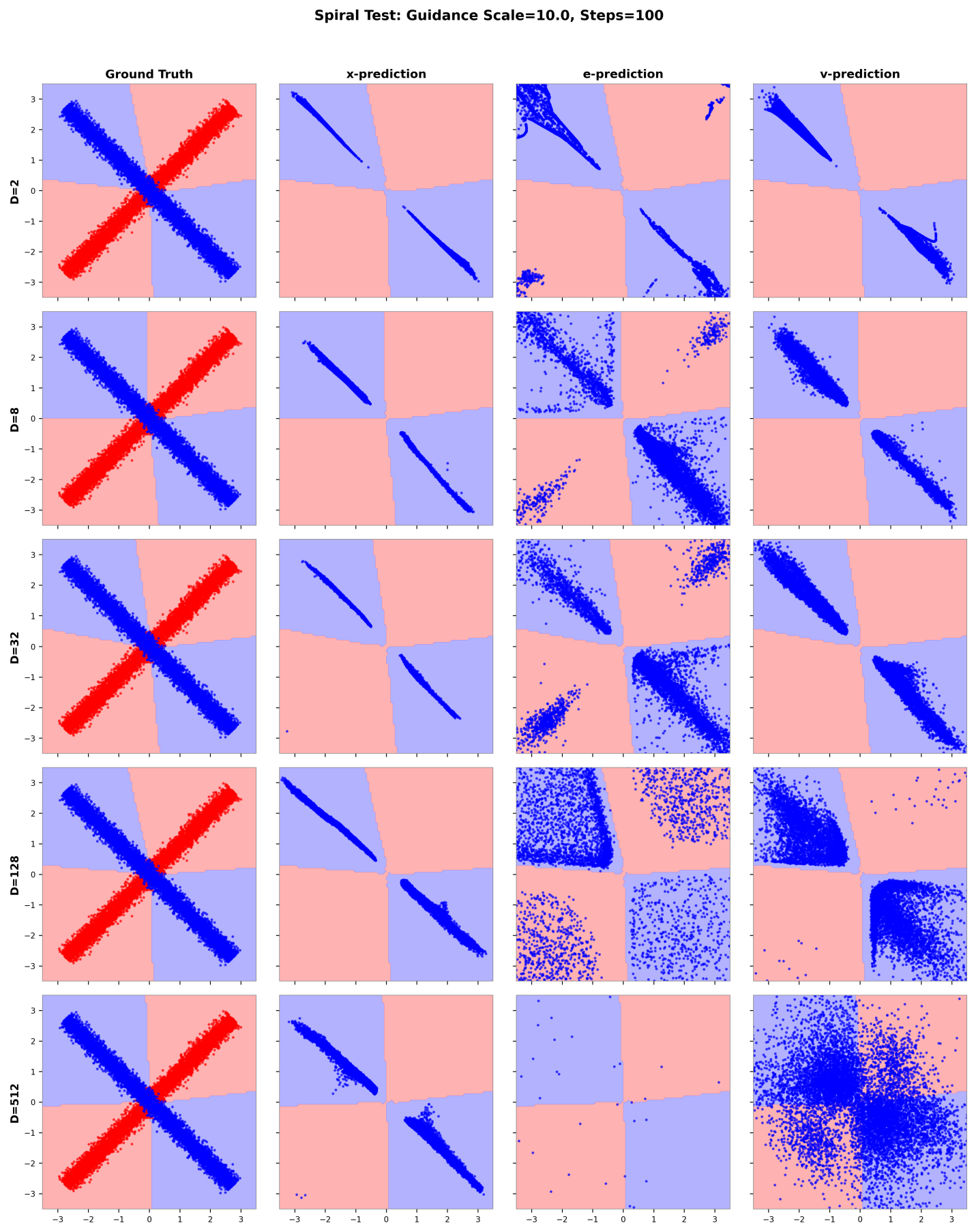}
    \caption{
        \textbf{Crossed-lines guided generation ($s{=}10$, 100 steps).}
        Rows: ambient dimension $D \in \{2, 8, 32, 128, 512\}$.
        Columns: ground truth, $x$-prediction, $\epsilon$-prediction, $v$-prediction.
        Background shading shows classifier decision boundaries.
        $x$-prediction preserves the line manifold across all dimensions; $\epsilon$-prediction collapses to scattered noise at high $D$; $v$-prediction shows intermediate degradation.
    }
    \label{fig:crossed_lines_grid}
\end{figure}

\paragraph{Half-arcs extension.}
\label{app:half_arcs_details}
To verify that the hierarchy generalizes beyond straight-line manifolds, we repeat the identical protocol on a \textbf{half-arcs} dataset: two semicircular arcs (upper and lower halves of a circle, 2 classes) with Gaussian noise $\sigma{=}0.1$, totaling 12{,}000 points. On-manifold rate is measured as the fraction of generated samples whose radial deviation from the arc is within the 95th percentile of ground truth deviations. \cref{tab:half_arcs_onmanifold} reports on-manifold rates across dimensions, \cref{tab:half_arcs_full} gives full metrics, and \cref{fig:half_arcs_grid} visualizes the generated distributions. The results confirm that $x \gg v \gg \epsilon$ is not an artifact of straight-line geometry: $x$-prediction maintains $>$85\% on-manifold rate even at $D{=}512$, while $\epsilon$-prediction collapses to $<$1\% and $v$-prediction degrades to 11.2\% (cf.\ 21.5\% on crossed-lines).

\begin{table}[tbp]
    \centering
    \caption{\textbf{Half-arcs: on-manifold rate (\%) under guidance ($s{=}10$).} Same hierarchy as crossed-lines: $x \gg v \gg \epsilon$.}
    \label{tab:half_arcs_onmanifold}
    \footnotesize
    \begin{tabular}{@{}lccccc@{}}
        \toprule
        Target & $D{=}2$ & $D{=}8$ & $D{=}32$ & $D{=}128$ & $D{=}512$ \\
        \midrule
        $\epsilon$-pred & 21.7 & 30.9 & 32.4 & 3.4 & 0.0 \\
        $v$-pred & 79.3 & 81.7 & 61.4 & 36.3 & 11.2 \\
        $x$-pred & \textbf{100} & \textbf{98.1} & \textbf{96.7} & \textbf{95.2} & \textbf{85.8} \\
        \bottomrule
    \end{tabular}
\end{table}

\begin{table}[tbp]
    \centering
    \caption{\textbf{Half-arcs full metrics ($s{=}10$, 100 steps, 10{,}000 samples).}}
    \label{tab:half_arcs_full}
    \footnotesize
    \begin{tabular}{@{}llccccc@{}}
        \toprule
        Metric & Target & $D{=}2$ & $D{=}8$ & $D{=}32$ & $D{=}128$ & $D{=}512$ \\
        \midrule
        \multirow{3}{*}{Target MMD ($\downarrow$)}
        & $\epsilon$ & .118 & .148 & .093 & .272 & .331 \\
        & $v$ & .028 & .038 & .016 & .020 & .133 \\
        & $x$ & \textbf{.044} & \textbf{.073} & \textbf{.018} & \textbf{.024} & \textbf{.009} \\
        \midrule
        \multirow{3}{*}{Class Acc.\ (\%)}
        & $\epsilon$ & 62.3 & 64.1 & 77.0 & 47.9 & 51.5 \\
        & $v$ & 88.7 & 95.2 & 88.9 & 89.9 & 65.0 \\
        & $x$ & \textbf{100} & \textbf{98.1} & \textbf{98.3} & \textbf{95.2} & \textbf{87.5} \\
        \bottomrule
    \end{tabular}
\end{table}

\begin{figure}[tbp]
    \centering
    \includegraphics[width=\textwidth]{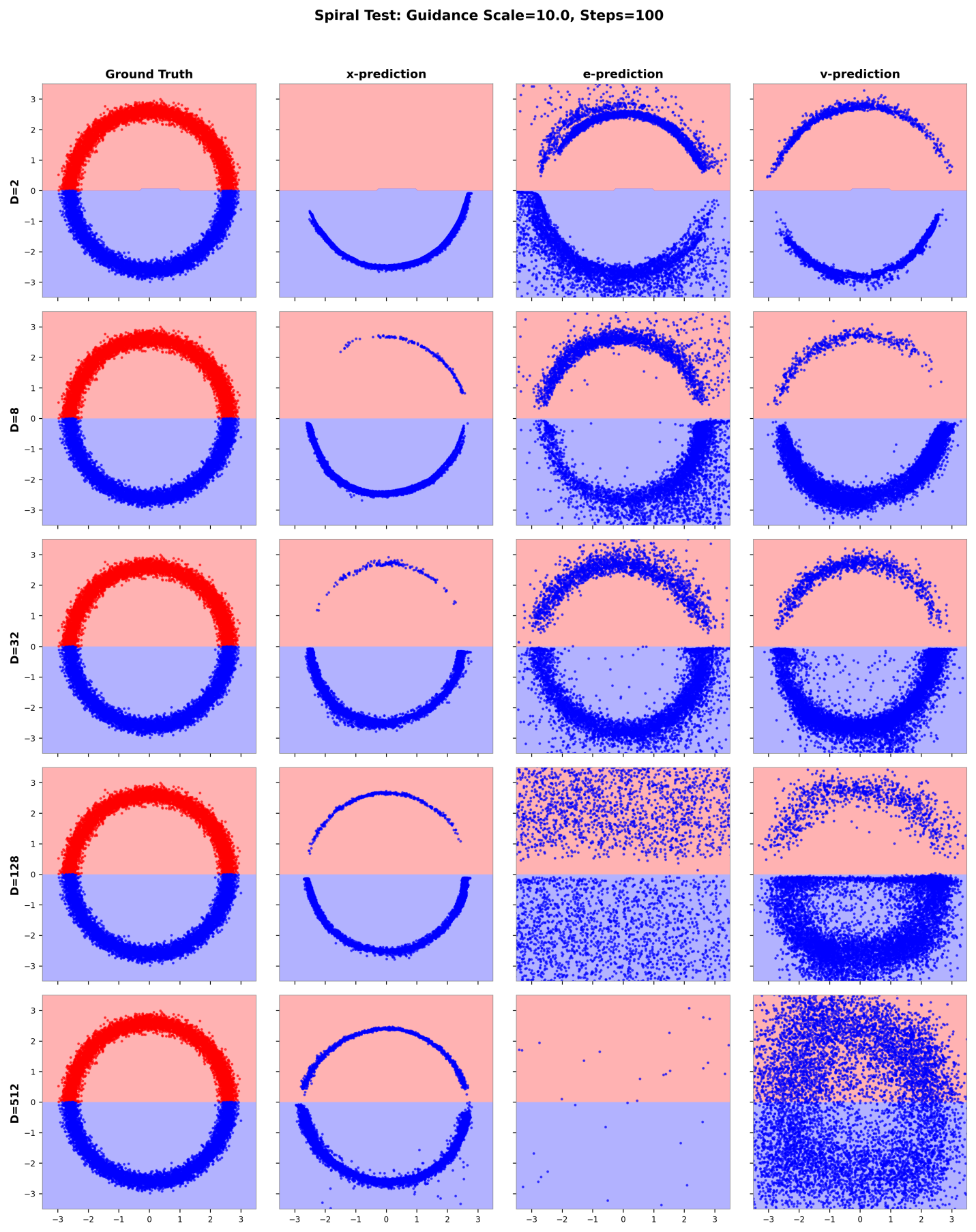}
    \caption{
        \textbf{Half-arcs guided generation ($s{=}10$, 100 steps).}
        Same layout as \cref{fig:crossed_lines_grid}. The curved manifold produces the same hierarchy: $x$-prediction preserves the arc structure across all $D$, while $\epsilon$-prediction collapses.
    }
    \label{fig:half_arcs_grid}
\end{figure}

\subsection{Choice of Guidance Method: DPS}
\label{app:why_dps}

Our goal is to characterize the fundamental relationship between prediction targets and gradient-based guidance quality. This is a scientific question (which prediction target produces the most faithful $\xhat$ estimates and, consequently, the most reliable guidance gradients?), not an engineering question of which hyperparameter configuration yields the best practical results. DPS~\cite{chung2023dps} is the natural choice precisely because it is the most basic form of gradient-based guidance: a single gradient step $\rho \nabla_{z_t} \log p(y \mid \xhat)$ per denoising step, with no auxiliary mechanisms.

\paragraph{Minimal confounding.}
DPS isolates the prediction target effect through a single free parameter $\rho$, the guidance strength. The TFG framework~\cite{ye2024tfg} generalizes DPS with mean guidance ($\mu$), Monte Carlo smoothing ($\sigma$), iteration ($N_{\text{iter}}$), and recurrence ($N_{\text{recur}}$), but each additional mechanism introduces its own interaction with the underlying $\xhat$ estimate. Under full TFG, it becomes unclear whether performance differences arise from the prediction target itself or from how well each target responds to a particular combination of auxiliary corrections. By stripping guidance to its essential form, DPS ensures that observed differences are attributable to the prediction target's gradient quality.

\paragraph{Consistency with theoretical analysis.}
Our theoretical framework (\cref{prop:error_amplification,thm:gradient_stability,prop:cumulative_error}) analyzes the error in $\xhat$ and its propagation through a single guidance gradient per step, exactly the DPS setting. Using full TFG would require extending the analysis to account for iterated gradient corrections, Monte Carlo averaging, and recurrence, which would obscure rather than illuminate the core prediction target effect.

\subsection{Sampling Procedure Details}

\cref{tab:inference_hyperparams} summarizes the inference configuration for each model.

\begin{table}[tbp]
    \caption{\textbf{Inference hyperparameters.} Paper defaults use each model's published lowest-FID configuration. All experiments use the actual configuration (bottom).}
    \label{tab:inference_hyperparams}
    \centering
    \footnotesize
    \begin{tabular}{@{}lccccc@{}}
        \toprule
        Model & Sampler & Steps & NFE & CFG & VAE \\
        \midrule
        \multicolumn{6}{@{}l}{\textit{Paper defaults (lowest FID):}} \\
        DiT-XL/2 & DDPM & 250 & 250 & 1.5 & ema \\
        SiT-XL/2 & Heun & 125 & 250 & 1.5 & ema \\
        JiT-H/16 & Heun & 50 & 100 & 2.2 & --- \\
        PixelFlow & Euler & 30$\times$4 & 120 & 2.4 & --- \\
        \midrule
        \multicolumn{6}{@{}l}{\textit{Actual (NFE$\approx$100):}} \\
        DiT-XL/2 & DDPM & 100 & 100 & 1.5 & ema \\
        SiT-XL/2 & Heun & 50 & 100 & 1.5 & ema \\
        JiT-H/16 & Heun & 50 & 100 & 2.2 & --- \\
        PixelFlow & Euler & 30$\times$4 & 120 & 2.4 & --- \\
        \bottomrule
    \end{tabular}
\end{table}

\paragraph{Model-Specific Samplers.}
\begin{itemize}
    \item \textbf{DiT (DDPM)}: DDPM with 100 steps (NFE$=$100)
    \item \textbf{SiT (Flow Matching)}: Heun with 50 steps (NFE$=$100)
    \item \textbf{JiT (Flow Matching)}: Heun with 50 steps (NFE$=$100)
    \item \textbf{PixelFlow (Flow Matching)}: Euler with 30 steps$\times$4 stages (NFE$=$120)
\end{itemize}

\paragraph{Heun Sampler.}
For Flow Matching models, the Heun (2nd-order Runge-Kutta~\cite{hairer1993ode}) update is:
\begin{align}
    v_1 &= \frac{F_\theta(x_t, t) - x_t}{1-t}, \quad
    x_{\text{euler}} = x_t + v_1 \cdot \Delta t \\
    v_2 &= \frac{F_\theta(x_{\text{euler}}, t+\Delta t) - x_{\text{euler}}}{1-t-\Delta t}, \quad
    x_{t+\Delta t} = x_t + \frac{v_1 + v_2}{2} \cdot \Delta t
\end{align}

\paragraph{NFE Calculation.}
\begin{itemize}
    \item \textbf{Euler + DPS}: $\text{NFE} = \text{steps}$ (one gradient evaluation per step)
    \item \textbf{Heun + DPS}: $\text{NFE} = \text{steps} \times 2$ (two function evaluations per step)
\end{itemize}

\subsection{Additional Qualitative Visualizations}
\label{app:qualitative_additional}

\cref{fig:qualitative_appendix} extends the qualitative comparison of \cref{fig:qualitative} to SiT and PixelFlow. SiT shows less mode collapse than DiT under strong guidance, consistent with $v$-prediction's intermediate error amplification. PixelFlow degrades more noticeably despite sharing JiT's pixel operating space.

\begin{figure}[tbp]
    \centering
    % (a) SiT moderate guidance
    \begin{subfigure}[b]{\textwidth}
        \includegraphics[width=\textwidth]{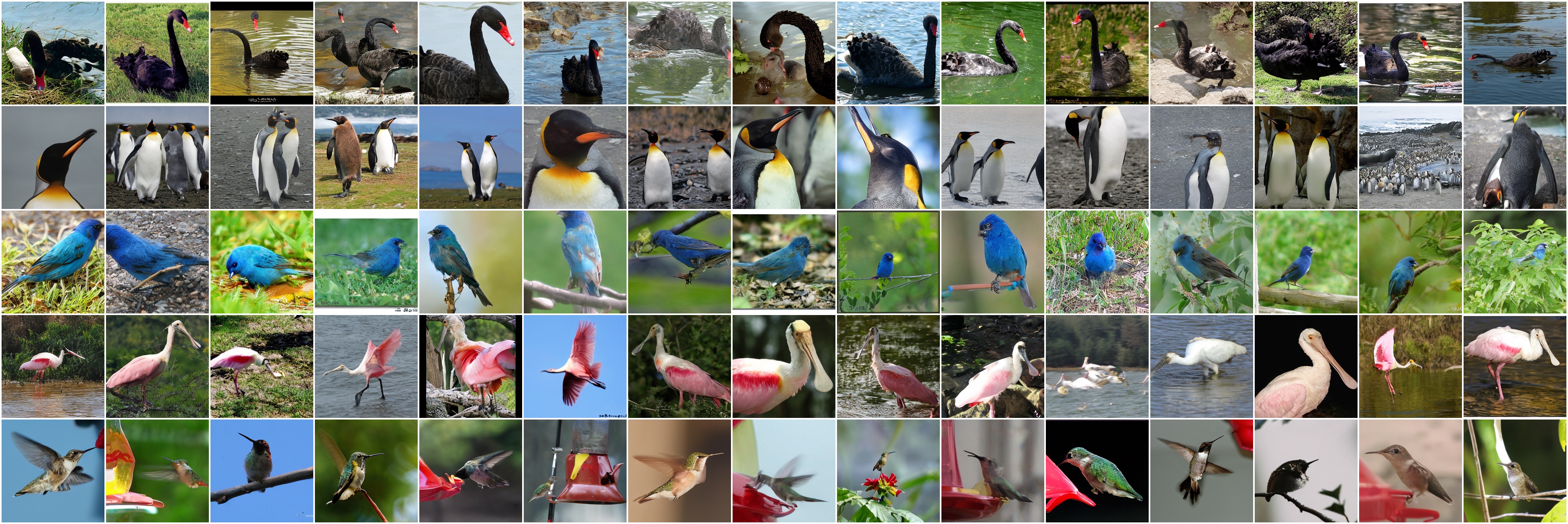}
        \caption{SiT ($v$, $\rho{=}0.05$): P-FID 5.8, C-FID 38.5, Val.\ 20.7\%}
        \label{fig:vis_on_sit}
    \end{subfigure}

    \vspace{1pt}

    % (b) PixelFlow moderate guidance
    \begin{subfigure}[b]{\textwidth}
        \includegraphics[width=\textwidth]{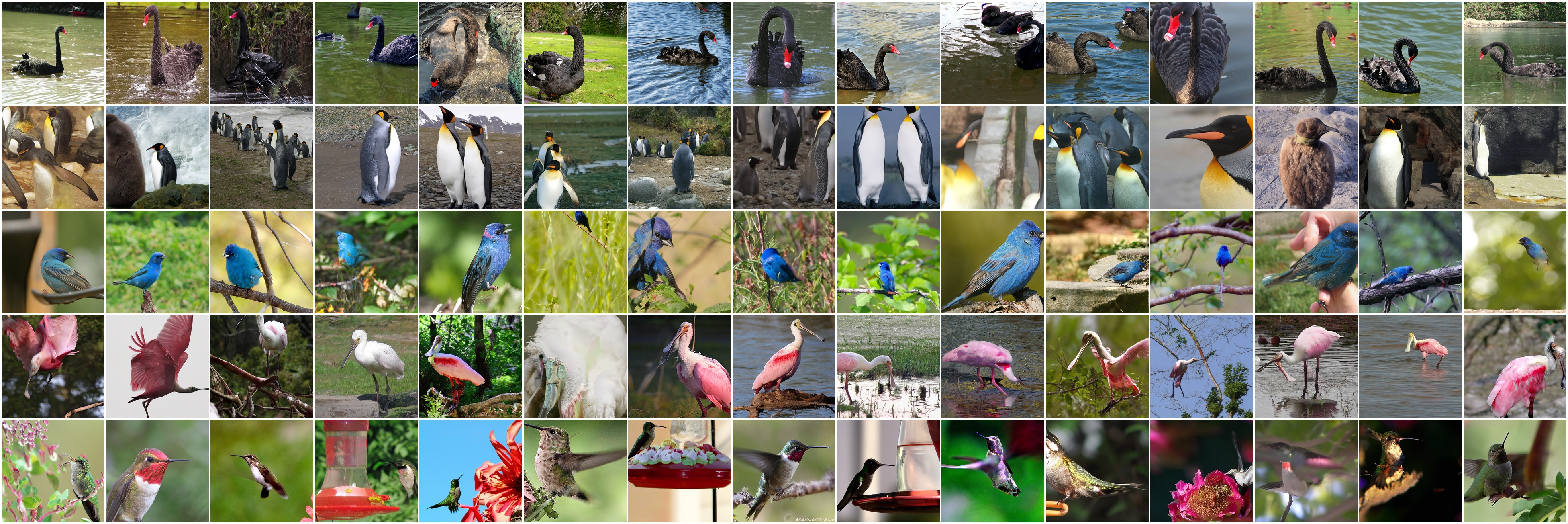}
        \caption{PixelFlow ($v$, $\rho{=}0.25$): P-FID 6.1, C-FID 38.8, Val.\ 19.4\%}
        \label{fig:vis_on_pixelflow}
    \end{subfigure}

    \vspace{1pt}

    % (c) SiT strong guidance
    \begin{subfigure}[b]{\textwidth}
        \includegraphics[width=\textwidth]{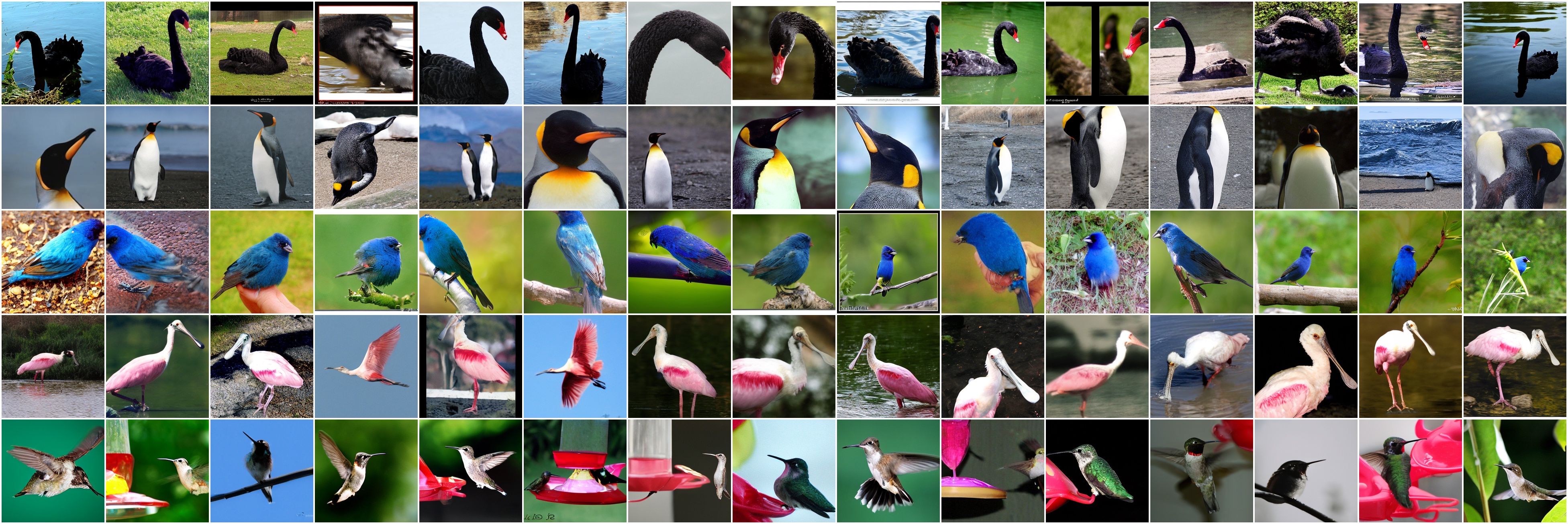}
        \caption{SiT ($v$, $\rho{=}2$): P-FID 12.5, C-FID 34.9, Val.\ 27.6\%}
        \label{fig:vis_off_sit}
    \end{subfigure}

    \vspace{1pt}

    % (d) PixelFlow strong guidance
    \begin{subfigure}[b]{\textwidth}
        \includegraphics[width=\textwidth]{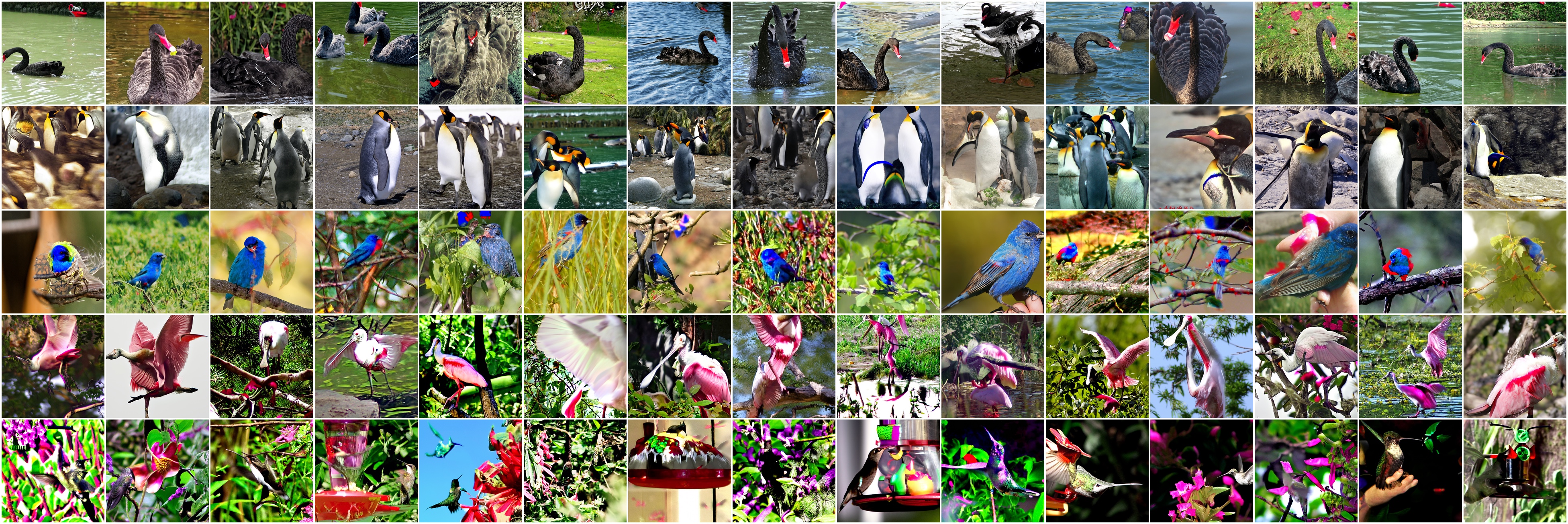}
        \caption{PixelFlow ($v$, $\rho{=}2$): P-FID 11.7, C-FID 36.2, Val.\ 20.1\%}
        \label{fig:vis_off_pixelflow}
    \end{subfigure}

    \caption{
        \textbf{Guided generation for SiT and PixelFlow (15 random samples each, no cherry-picking).}
        (a),(b): moderate guidance; (c),(d): strong guidance. Same five species as \cref{fig:qualitative}: Black Swan, Emperor Penguin, Painted Bunting, Roseate Spoonbill, Ruby-throated Hummingbird.
    }
    \label{fig:qualitative_appendix}
\end{figure}

% =============================================================================
% APPENDIX J: PRDC PRECISION–RECALL ANALYSIS
% =============================================================================
\section{Precision--Recall Analysis}
\label{app:prdc}

The main text evaluates guidance quality through FID-based metrics (P-FID, C-FID) and classifier-based Validity. Here we complement those results with \emph{manifold-aware} Precision and Recall curves that directly measure fidelity--diversity trade-offs across guidance strengths.

\subsection{Metric Definition}

We adopt the $k$-nearest-neighbour Precision and Recall of Naeem~et~al.~\cite{naeem2020reliable}, computed in DINOv2 ViT-B/14 feature space~\cite{oquab2023dinov2} (768-dimensional CLS tokens, $k{=}5$). Let $\Phi_r = \{\phi(x_i)\}$ and $\Phi_g = \{\phi(g_j)\}$ denote the real and generated feature sets, and $\mathrm{NN}_k(\phi, S)$ the distance to the $k$-th nearest neighbour of $\phi$ in set $S$. A sample $\phi$ is considered to lie \emph{within the manifold} of $S$ if $\|\phi - \mathrm{NN}_1(\phi, S)\| \le \mathrm{NN}_k(\mathrm{NN}_1(\phi, S), S)$, i.e., its nearest neighbour in $S$ has $\phi$ inside its own $k$-NN ball. Then:
\begin{align}
    \text{Precision} &= \frac{1}{|\Phi_g|}\sum_{g \in \Phi_g} \mathbf{1}[g \text{ within manifold of } \Phi_r], \label{eq:prdc_precision} \\
    \text{Recall} &= \frac{1}{|\Phi_r|}\sum_{r \in \Phi_r} \mathbf{1}[r \text{ within manifold of } \Phi_g]. \label{eq:prdc_recall}
\end{align}
Precision measures the fraction of generated samples that fall within the real data manifold (\emph{fidelity}), while Recall measures the fraction of real samples covered by the generated distribution (\emph{diversity}). We use DINOv2 rather than Inception-v3 as the feature extractor because DINOv2's self-supervised features better capture fine-grained visual similarity relevant to our bird species benchmark~\cite{oquab2023dinov2}.

\subsection{Results}

\cref{fig:prdc_sweep} traces Precision and Recall over the same DPS $\rho$-sweep as the main experiments (\cref{subsec:bird_results}), extending the four-model mode-collapse figure (\cref{fig:prdc_main}) to all six models. Three observations are specific to the wider lineup.

\begin{figure}[tbp]
    \centering
    \includegraphics[width=0.75\textwidth]{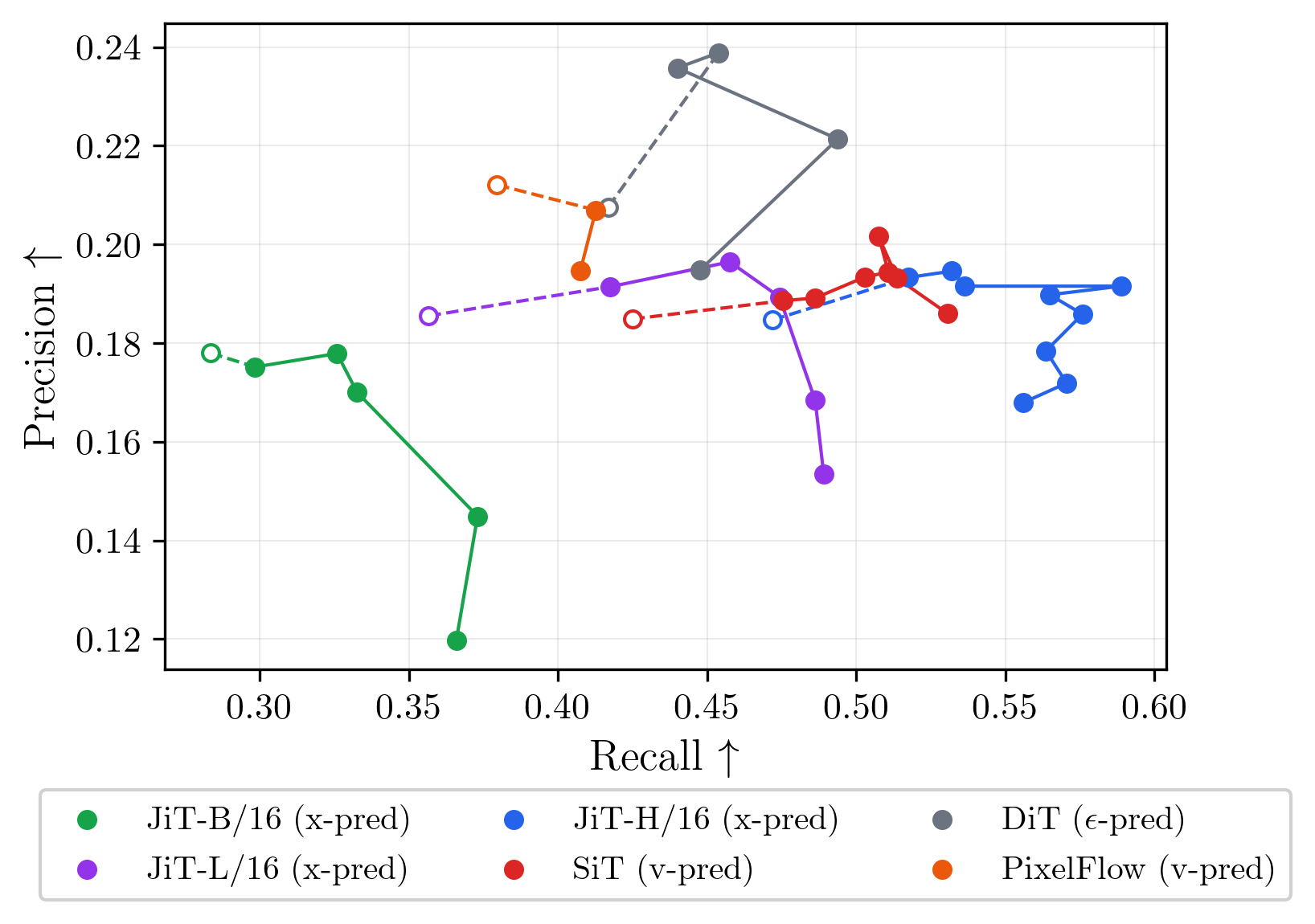}
    \caption{
        \textbf{Precision vs.\ Recall on fine-grained bird classification across guidance strengths.}
        Each curve traces one model as DPS strength $\rho$ increases; open markers denote CFG-only baselines ($\rho{=}0$), dashed segments connect to the first guided point.
        Axes share identical scale (each grid cell = $0.05 \times 0.05$).
        The upper-right region is preferred (high fidelity and high diversity).
    }
    \label{fig:prdc_sweep}
\end{figure}

\paragraph{Capacity expands coverage, not sharpness.}
Across the JiT family, Recall improves steadily with scale (JiT-B 0.37 $\to$ JiT-L 0.49 $\to$ JiT-H 0.59) while Precision stays in a narrow band ($\sim$0.17--0.19). Added capacity broadens coverage of the target distribution rather than sharpening individual samples.

\paragraph{$\epsilon$-prediction shows the mode-collapse signature.}
DiT attains the highest Precision (0.24 at $\rho{=}0.05$) of any model, but its Recall never exceeds ${\sim}$0.49, well below JiT-H's 0.59. Under stronger guidance its Precision erodes (0.19 at $\rho{=}0.5$) while Recall stays low (${\sim}$0.45): the joint high-Precision, low-Recall pattern that \cref{subsec:bird_results} identifies as the mode-collapse signature of $\epsilon$-prediction.

\paragraph{$v$-prediction gains little diversity from guidance.}
PixelFlow ($v$, pixel) and SiT ($v$, latent) remain in the Recall range 0.40--0.53, below JiT-H at every operating point. PixelFlow's Precision drifts down (0.21 $\to$ 0.19) with Recall essentially unchanged (0.40--0.42), so guidance degrades its fidelity without adding diversity; SiT behaves similarly.

These sample-level observations agree with the distribution-level C-FID analysis and with the prediction that $x$-prediction preserves manifold proximity under guidance (\cref{thm:gradient_stability}).

% =============================================================================
% APPENDIX H: REPRODUCIBILITY STATEMENT
% =============================================================================

\section{Reproducibility Statement}
\label{app:reproducibility}

\paragraph{Pretrained Models.}
All models are publicly available from official sources:
\begin{itemize}[leftmargin=*,itemsep=2pt,parsep=0pt]
    \item DiT-XL/2: \url{https://github.com/facebookresearch/DiT}
    \item SiT-XL/2: \url{https://github.com/willisma/SiT}
    \item JiT-B/L/H: \url{https://github.com/LTH14/JiT}
    \item Stable Diffusion VAE: \url{https://huggingface.co/stabilityai/sd-vae-ft-mse}
\end{itemize}

\paragraph{Code and experimental details.}
We provide DPS guidance code for all evaluated models, evaluation scripts for all reported metrics, and hyperparameter configurations for all experiments. All inference hyperparameters are specified in \cref{tab:inference_hyperparams} and throughout \suppref{app:experimental_protocols}, and each configuration is evaluated once with a fixed random seed (42).

\paragraph{Compute Resources.}
Using pretrained models eliminates training compute requirements:
\begin{itemize}[leftmargin=*,itemsep=2pt,parsep=0pt]
    \item DPS guidance experiments: $\sim$2--4 A100 GPU-hours per model/task combination
    \item Evaluation (FID-50K): $\sim$1 A100 GPU-hour per configuration
    \item Total estimated compute: $<$30 A100 GPU-hours
\end{itemize}
Per-step cost differs by operating space: pixel-space models backpropagate guidance gradients through higher-dimensional states ($256{\times}256{\times}3{=}196{,}608$) than latent models ($32{\times}32{\times}4{=}4{,}096$), but latent models incur a VAE decoder forward and backward pass at every guidance step to compute pixel-space gradients, partially offsetting this gap.

\paragraph{Datasets.}
We use standard public datasets: ImageNet-1K~\cite{deng2009imagenet} for all image experiments, with the standard validation set for FID computation. The fine-grained bird benchmark uses the 525 Bird Species dataset~\cite{piosenka2023birds} (CC0: Public Domain, available on Kaggle and HuggingFace; see \suppref{app:experimental_protocols} for details). The 143-species-to-ImageNet mapping file is released with our code. Both guidance and evaluation classifiers are publicly available on HuggingFace. No proprietary or restricted-access datasets are used.

\fi

\end{document}